\documentclass[]{IEEEtran}

\IEEEoverridecommandlockouts                              
\usepackage{url}
\usepackage{amsmath,amssymb}
\usepackage{graphicx,subfigure}
\usepackage{ntheorem}
\usepackage{color}
\usepackage{breakurl}
\usepackage{cite}
\usepackage{wasysym}
\usepackage[ruled,linesnumbered]{algorithm2e}
\usepackage{todonotes}
\usepackage{stmaryrd}
\mathchardef\mhyphen="2D 

\newcommand{\eg}{\textit{e.g.}}
\newcommand{\ie}{\textit{i.e.}}

\DeclareMathOperator*{\bigO}{\mathcal{O}}
\DeclareMathOperator*{\minimize}{minimize}

\DeclareMathOperator*{\argmin}{\arg\min}
\DeclareMathOperator*{\VP}{\mathbf{VP}}

\DeclareMathOperator*{\polylog}{\text{polylog}}

\theoremclass{Theorem}

\newtheorem{lemma}{Lemma}
\newtheorem{proposition}{Proposition}
\newtheorem{theorem}{Theorem}

\theoremclass{Problem}
\newtheorem{problem}{Problem}

\theoremclass{Definition}
\newtheorem{definition}{Definition}

\theoremclass{Remark}

\title{\LARGE\bf Algorithms for Visibility-Based Monitoring with Robot Teams}%
\author{Pratap Tokekar, Ashish Kumar Budhiraja and Vijay Kumar%
\thanks{P. Tokekar and A. Budhiraja are with the Department of Electrical \& Computer Engineering, Virginia Tech, Blacksburg VA, USA. \texttt{\small \{tokekar, ashishkb\}@vt.edu}}
\thanks{V. Kumar is with the Department of Mechanical Engineering
  and Applied Mechanics, University of Pennsylvania, Philadelphia, PA,
  USA. \texttt{\small kumar@seas.upenn.edu}.}%
\thanks{This material is based upon work supported in part by the National Science Foundation under Grant No. 1566247, ONR grant N00014-09-1-1051, NIFA grant 2015-67021-23857.}}

\begin{document}
\maketitle

\begin{abstract}
We study the problem of planning paths for a team of robots for visually monitoring an environment. Our work is motivated by surveillance and persistent monitoring applications. We are given a set of target points in a polygonal environment that must be monitored using robots with cameras. The goal is to compute paths for all robots such that every target is visible from at least one path. In its general form, this problem is NP-hard as it generalizes the Art Gallery Problem and the Watchman Route Problem. We study two versions: (i) a geometric version in \emph{street polygons} for which we give a polynomial time $4$--approximation algorithm; and (ii) a general version for which we present a practical solution that finds the optimal solution in possibly exponential time. In addition to theoretical proofs, we also present results from simulation studies.
\end{abstract}

\section{Introduction} \label{sec:intro}
We study the problem of planning paths for a team of robots tasked with visually monitoring complex environments.
Visibility-based monitoring problems commonly occur in many applications such as surveillance, infrastructure
inspection~\cite{ozaslaninspection}, and environmental monitoring~\cite{smith2011persistent}. These problems have received
significant interest recently~\cite{smith2012persistent,yu2014correlated,lan2013planning}, thanks in part, to the technological advances that have made it easy to rapidly deploy teams of robots capable of performing such tasks. For example, Michael et al.~\cite{michael2011persistent} demonstrated the feasibility of carrying out persistent monitoring tasks with a team of Unmanned Aerial Vehicles (UAVs) with onboard
cameras.

\begin{figure}[htb]
\centering{ 
\includegraphics[width=0.9\columnwidth]{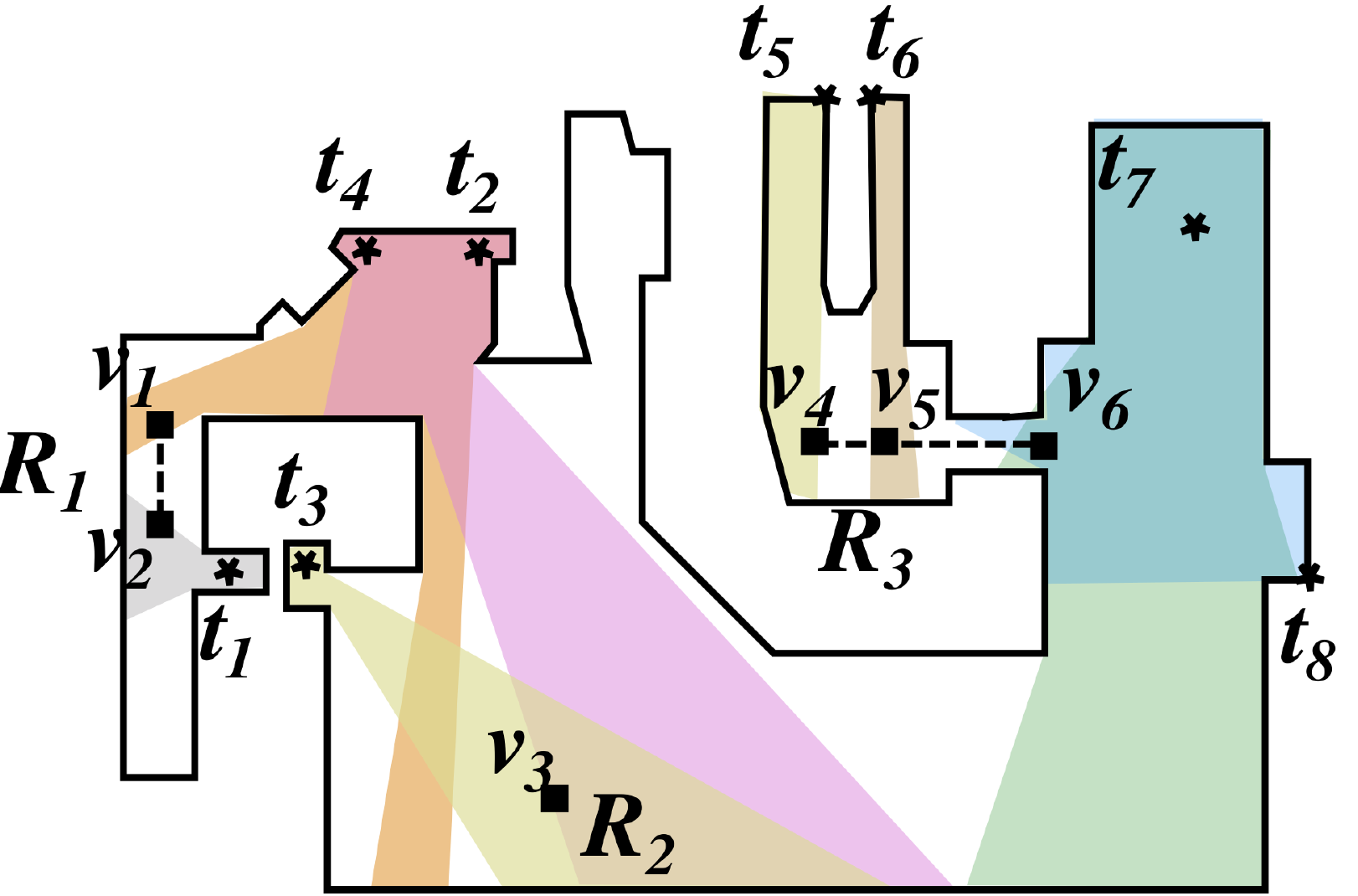}
\caption{Problem formulation. We are given a set of target points ($\star$) in a polygonal environment. Our goal is to find $m$ paths
   and discrete viewpoints on the paths ($\blacksquare$) such that each target is seen from at least one viewpoint. Here, $R_1$ sees $t_1$ and $t_2$ along its path from $v_1 \rightarrow v_2$; $R_2$ remains stationary at $v_3$ from where it sees $t_3$ and $t_4$; and $R_3$ sees $t_5$--$t_8$ from its path $v_4 \rightarrow v_5 \rightarrow v_6$. The cost of a path is a weighted combination of its length (travel time) and the number of viewpoints (measurement time). A path may consist of only one point (as is the case for $R_2$). The objective is to minimize the maximum path cost.\label{fig:motivation}}}
\end{figure}

Persistent monitoring problems are typically studied when the points of interest are given as input. The points may have associated weights representing their importance. A common objective is to find the order of visiting the points that minimizes the weighted latency. Alamdari et al.~\cite{alamdari2014persistent} showed that this problem is NP-hard and presented two $\log$ factor approximation algorithms. In many settings, the path to be followed by the robots is given as input as well, and the speed of the robot must be optimized to minimize the maximum weighted latency. Cassandras et al.~\cite{cassandras2013optimal} presented an optimal control approach to determine the speed profiles for multiple robots when their motion is constrained to a given curve. Yu et al.~\cite{yu2014persistent} presented an optimal solution for computing speed profiles for a single robot moving along a closed curve to sense the maximum number of stochastically arriving events on a curve. Pasqualetti et al.~\cite{pasqualetti2012cooperative} presented distributed control laws for coordination between multiple robots patrolling on a metric graph.

We consider a richer version of the problem where the points to be visited by the robots are not given, and instead must be computed based on visibility-based sensing. We are given a set of target points in a polygonal environment. Each robot carries a camera and can see any target as long as the straight line joining them is not obstructed by the boundary of the environment. Our goal is to compute paths for $m$ robots, so as to ensure that each target is seen from at least one point on some path. Figure~\ref{fig:motivation} shows an example scenario for $m=3$.

	Our problem is a generalization of the Art Gallery Problem (AGP)~\cite{o1987art} and the Watchman Route Problem (WRP)~\cite{chin1988optimum}. The objective in AGP is to find the smallest set of ``guard'' locations, such that every point in an input polygon is seen from at least one guard. AGP is NP-hard for most types of input polygons~\cite{o1987art}, and very few approximation algorithms exist even for special cases. The objective in WRP is to find a tour of minimum length for a single robot (i.e., watchman) so as to see every point in an input polygon. There is an optimal algorithm for solving WRP in polygons without any holes~\cite{carlsson1999finding} and a $\bigO(\log^2 n)$ approximation algorithm for $n$-sided polygons with holes~\cite{mitchell2013approximating}. Carlsson et al~\cite{carlsson1999finding} introduced $m$--WRP where the goal is to find $m$ tours such that each point in the environment is seen from at least one tour. The objective is to minimize the total length of $m$ tours. They showed that the problem is NP-hard (in fact, no approximation guarantee is possible).

Using the length of a tour as the cost is reasonable when a robot is capable of obtaining images as it is moving. However, in practice, obtaining high-resolution images while moving may lead to motion blur or cause artifacts to appear due to rolling shutter cameras. This is especially the case when UAVs are to be used. It would be desirable for the robot to stop to obtain a measurement. Instead of finding a continuous path, we would like to find a set of discrete viewpoints on $m$ paths. The cost of a path can be modeled as the weighted sum of the length of the path (travel time) and the number of measurements along the path (measurement time). Wang et al.~\cite{wang2010generalized} first introduced this objective function for WRP for the case of a single robot and termed it the Generalized WRP (GWRP). They showed that GWRP is NP-hard and presented a $\bigO(\polylog n)$ approximation for the restricted case when each viewpoint is required to see a complete polygon edge.

We introduce the $m$ robot version of GWRP. This problem, in general, is NP-hard since it generalizes the NP-hard problems of GWRP and $m$-WRP. Hence, we consider special instances of the problem and present a number of positive results. 
In particular, we characterize the conditions under which the problem has an optimal algorithm (Section~\ref{sec:dynprog}) and a present a constant-factor approximation algorithm for a special class of environments (Section~\ref{sec:street}). In addition to theoretical analysis, we perform simulations to study the effect of the number of robots and targets on the optimal cost.

Finally, we study the case when the measurement time is zero (Section~\ref{sec:general}). We present a practical solution that finds paths for $m$ robots. Our solution can be applied for a broad class of environments (\eg, 2.5D, 3D) and can incorporate practical sensing constraints (\eg, limited sensing range and field-of-view). The added generality comes at the expense of running time. Instead of a polynomial time solution, our algorithm may take possibly exponential time. We show how to use existing, sophisticated Traveling Salesperson Problem solvers to produce solutions in reasonable amounts of time (for typical instances). 

This paper builds on the work presented at~\cite{tokekar2015visibility} which did not include the general solution in Section~\ref{sec:general}.  We formally state the various problems considered and our contributions in the following section.

\section{Problem Formulation} \label{sec:probform}
We study three problems for visibility-based persistent monitoring. The environment $P$ is an $n$-sided 2D polygon without holes (in one of the problems we consider a 1.5D terrain\footnote{A 1.5D terrain is defined by a continuous (but not necessarily differentiable) function, $f(x)$, which can be interpreted as the height of the terrain at a coordinate $x$. Figure~\ref{fig:terrain} gives an example.} environment). We are given a set of target points, $X$, within $P$. We have $m$ robots each carrying an omnidirectional camera. Let each robot travel with unit speed and let the time to obtain an image be $t_m$. Let $\Pi_i$ denote both the $i^{th}$ path and $V_i$ denote the discrete set of viewpoints along this path. Let $\Pi$ denote the collection of $m$ paths. Let $\VP(p)$ denote the visibility polygon of a point $p$ in $P$ and $\VP(V_i)$ denote the union of visibility polygons of all viewpoints on $\Pi_i$. We relax many of these assumptions for the general case in Section~\ref{sec:general}.



In the first problem, we are given a curve in the polygon along which we must determine the set of viewpoints. The curve can be, for example, the boundary of the environment for border patrolling, or a safe navigation path within the environment. The goal is to find $m$ paths along this curve along with viewpoints on the paths, to see every point in $X$. We show that if the set $X$ and the curve satisfy a property, termed \emph{chain-visibility}, then this problem can be solved optimally.

\begin{figure*}[thb]
\centering{
  \subfigure[\label{fig:shortest}]{\includegraphics[width=0.25\textwidth]{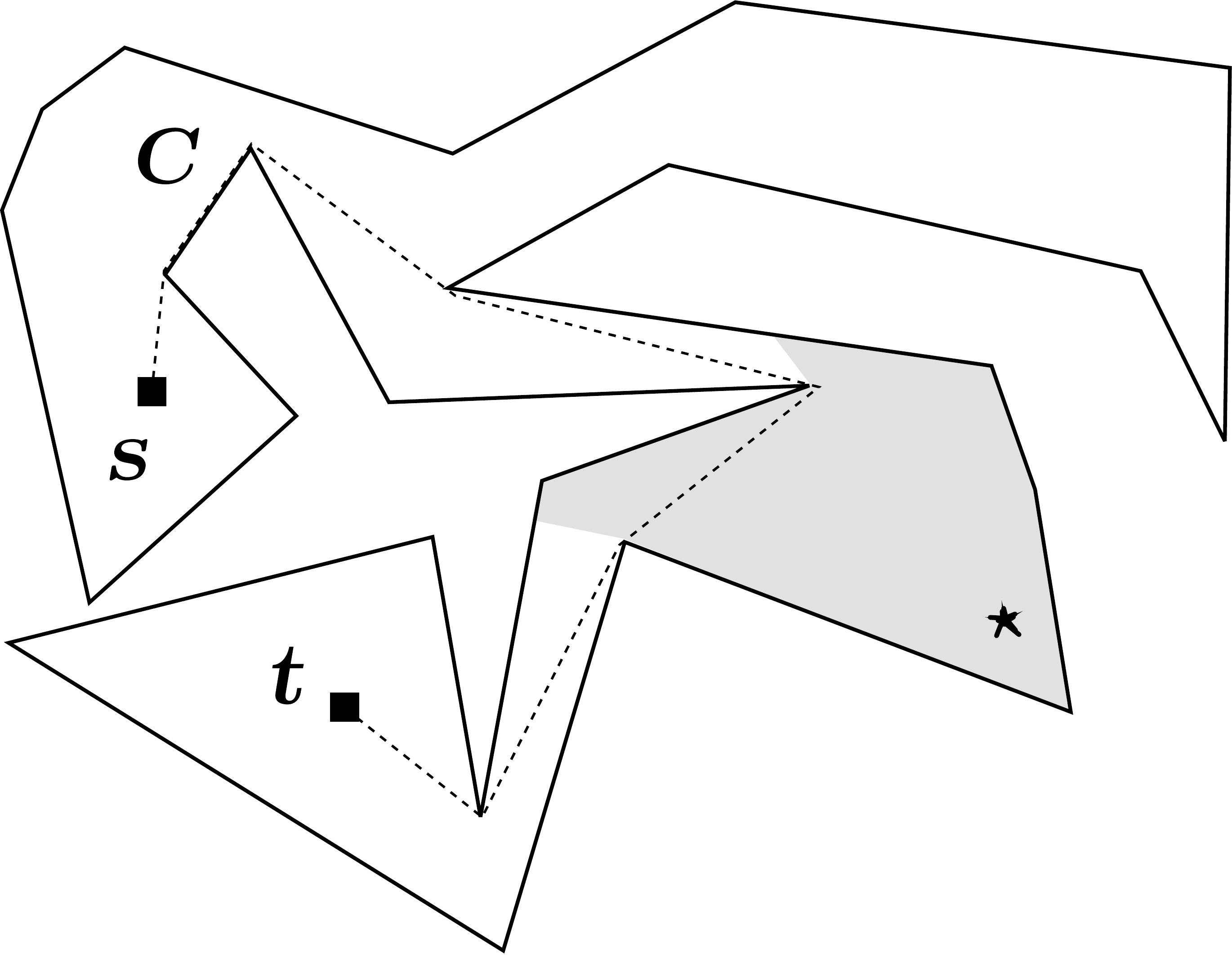}}
  \subfigure[\label{fig:street}]{\includegraphics[width=0.25\textwidth]{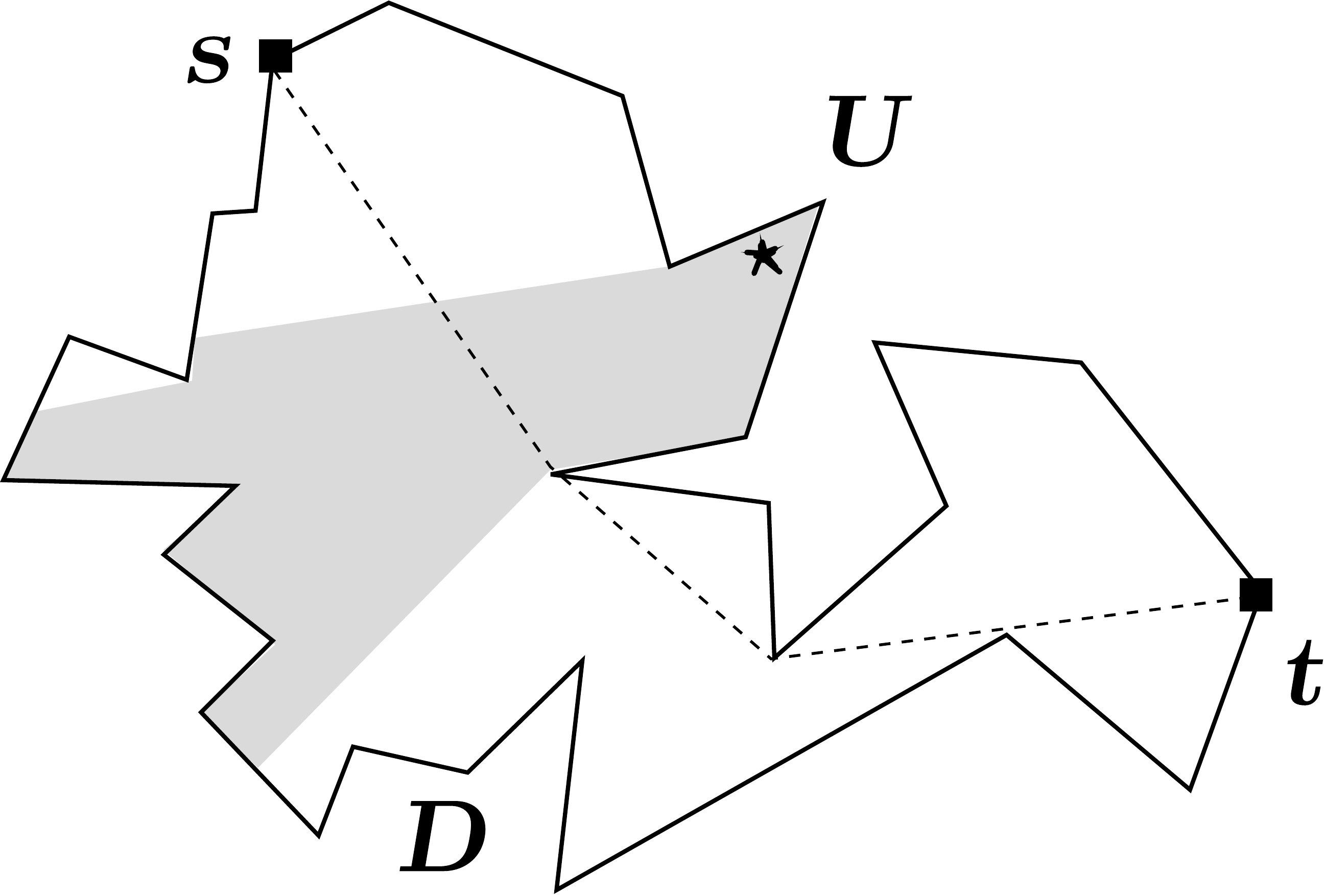}}
  \subfigure[\label{fig:terrain}]{\includegraphics[width=0.3\textwidth]{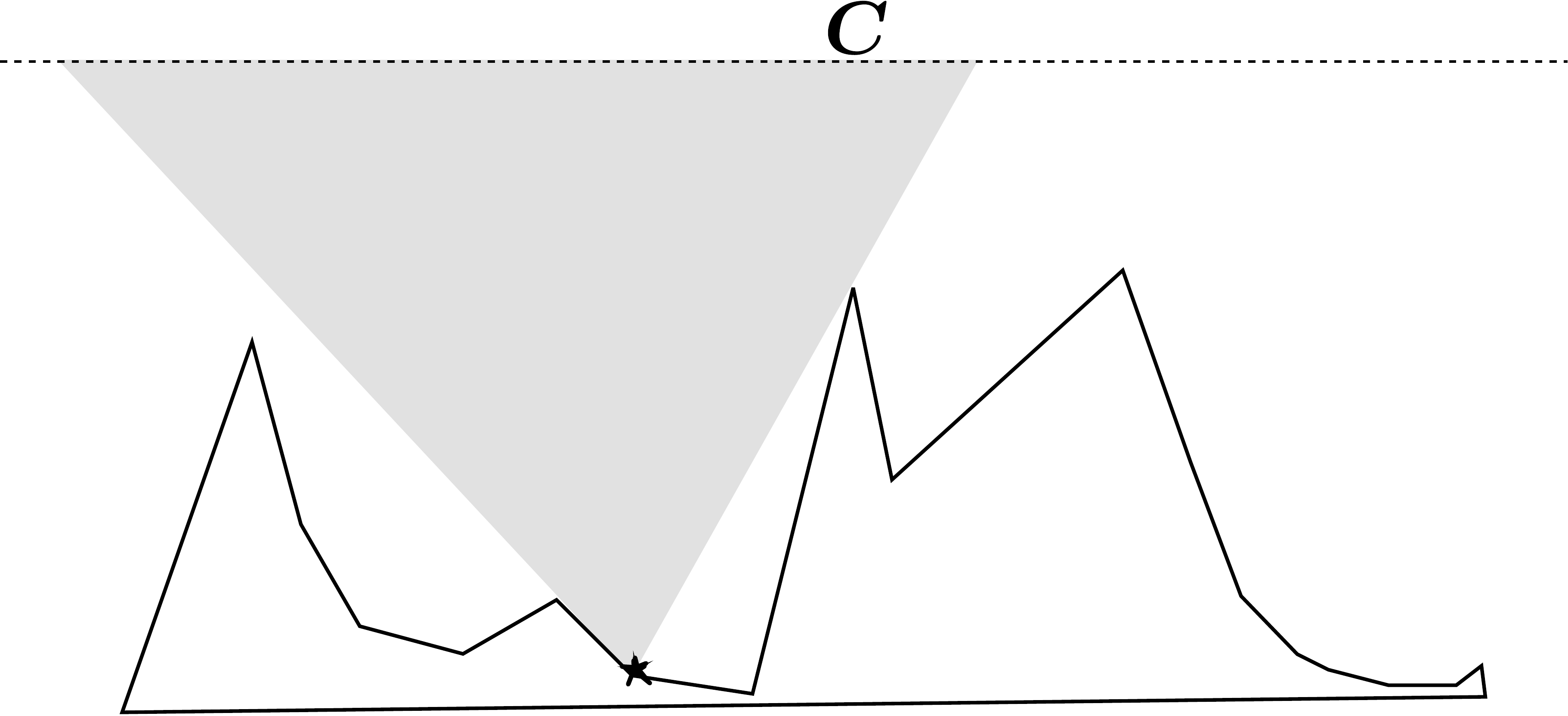}}
\caption{Examples of curves and environments satisfying the chain
  visibility property (Definition~\ref{def:chainvisibility}). (a)
  Shortest path between any two points, $s$ and $t$, in a polygon without holes. (b)
  A collapsed watchman route in a street polygon. (c) 1.5D terrain
  with a fixed altitude path.\label{fig:chainvisibleexamples}}}
\end{figure*}

\begin{definition}
Let $X$ be a set of points and $C$ be some curve in a 2D
environment. The pair $(X,C)$ is said to be chain visible if the
intersection of the visibility polygon of any point $x\in X$ with $C$,
i.e., $\VP(x)\cap X$ is either empty or a connected chain.
\label{def:chainvisibility}
\end{definition}
Although restrictive, the chain-visibility property is satisfied by
various curves. Figure~\ref{fig:chainvisibleexamples} shows some
examples. 

Chain-visibility was used by Carlsson and Nilsson~\cite{carlsson1999computing} to show that there always exists a \emph{collapsed watchman path} satisfying chain-visibility in street polygons. A street polygon is a polygon without holes with the property that its boundary can be partitioned into two chains, $U$ and $D$, and any point in $U$ is visible from some point in $D$ and vice versa. Figure~\ref{fig:street} gives an example. We give more examples of chain-visibility in the following proposition.

\begin{proposition}
Following pairs of target points of interest, $X$, and curves, $C$, all satisfy the chain-visibility property:
\begin{itemize}
\item $X$ is any set of points in a street polygon and $C$ is a collapsed watchman route.
\item $X$ is any set of points in a polygon without holes and $C$ is
  the shortest path between any pair of points $s$ and $t$ in the
  polygon. In particular, $C$ can be a straight line within the
  polygon.
\item $X$ is any set of points on a 1.5D terrain and $C$ is a fixed
  altitude path.
\end{itemize}
\end{proposition}

Naturally, if the intersection of the visibility region of a point with $C$ is empty, we will not be able to compute a viewpoint on $C$ to see the point. Hence, we consider only those situations where each point in $X$ is visible from some point in $C$ while satisfying the chain-visibility property. Formally, the first problem we consider is the following:
\begin{problem}
~\\
\textbf{Input:} set of points of interest, $X$, and a curve, $C$, such
that $(X,C)$ is chain-visible\\
\textbf{Output:} $m$ paths $\Pi=\{\Pi_i|\Pi_i\subseteq C\}$ each with
viewpoints, $V_i$, where for all $x\in X$ we have $x\in\cup\VP(V_i)$.\\
\textbf{Objective:} $\minimize\limits_{\Pi}\max\limits_{\Pi_i\in\Pi} \left(l(\Pi_i) +
|V_i|t_m\right)$.\\
\textbf{Contribution:} Optimal algorithm
(Section~\ref{sec:dynprog}). 
\label{prob:dynprog}
\end{problem}
The optimal algorithm for this problem is given in Section~\ref{sec:dynprog}.

In the second problem, $X$ is simply the set of all points in the polygon. Our goal is to find $m$ paths, restricted to a chain-visible curve, so as to monitor every point in the environment. Here, we focus on the case of street polygons where we know there always exists a curve, namely the collapsed watchman route, which is chain-visible for any subset of points in the polygon~\cite{carlsson1999computing}. Street polygons have previously been studied in the context of robot
navigation in unknown environments~\cite{icking2004optimal,icking2002competitive}.

\begin{problem}
~\\
\textbf{Input:} street polygon, $P$, and a chain-visible curve, $C$\\
\textbf{Output:} $m$ paths, $\Pi=\{\Pi_i|\Pi_i\subseteq C\}$, each with
viewpoints, $V_i$, where for all $x\in X$ we have $x\in\cup\VP(V_i)$.\\
\textbf{Objective:} $\minimize\limits_{\Pi}\max\limits_{\Pi_i\in\Pi} \left(l(\Pi_i) +
|V_i|t_m\right)$.\\
\textbf{Contribution:} $4$--approximation algorithm
(Section~\ref{sec:street}). 
\label{prob:street}
\end{problem}
Carlsson and Nilsson~\cite{carlsson1999computing} presented an optimal algorithm to find the fewest number of viewpoints along $C$ to see every point in $P$. One way to compute paths would be to first find this smallest set of viewpoints and then distribute them into $m$ paths. Unfortunately, this approach can lead to paths which are arbitrarily longer and consequently arbitrarily worse than optimal paths (Figure~\ref{fig:badstreet}). Nevertheless, we present a $4$--approximation algorithm for Problem~\ref{prob:street}.

We consider a general version in the third problem at the expense of discretization. The cost is simply the sum of the distances traveled by the robots. We are given a graph along which the robots can navigate. The robots can stop and take a measurement at the vertices of this graph. We show how to find the optimal solution by reducing this to a TSP instance which can then solved by a numerical solver (\eg, concorde~\cite{applegate2006concorde}).

\begin{problem}
~\\
\textbf{Input:} set of points of interest, $X$, and a set of viewpoints, $V$, in a polygon\\
\textbf{Output:} $m$ paths, $\Pi=\{\Pi_i|\Pi_i\subseteq V\}$, where for all $x\in X$ we have $x\in\cup\VP(\Pi_i)$.\\
\textbf{Objective:} $\minimize\limits_{\Pi}\sum\limits_{\Pi_i\in\Pi} \left(l(\Pi_i)\right)$.\\
\textbf{Contribution:} optimal solution
(Section~\ref{sec:general}). 
\label{prob:general}
\end{problem}

The solution is presented in Section~\ref{sec:general}. Our solution extends the single robot solution presented by Obermeyer et al.~\cite{obermeyer2012sampling} to the case of multiple robots. 

\begin{figure}[htb]
\centering{ 
\includegraphics[width=0.8\columnwidth]{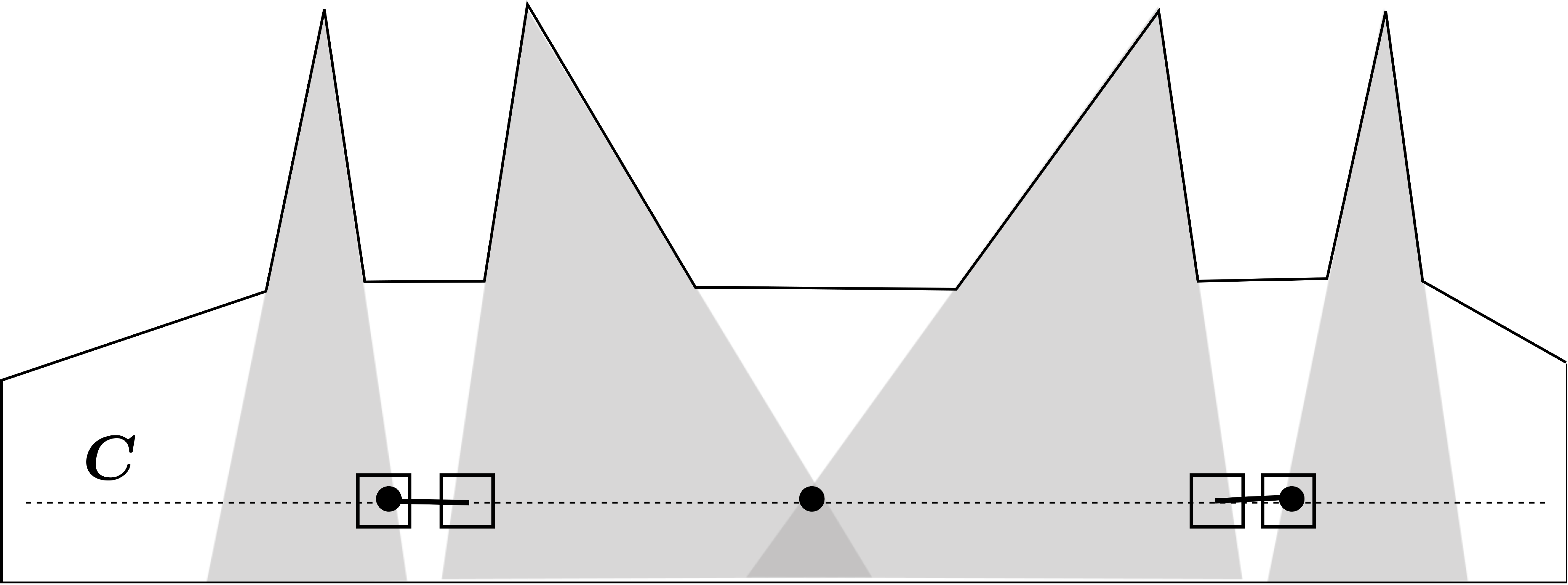}
\caption{Problem~\ref{prob:street}. The optimal solution with $m=2$
  paths consists of two short paths with two viewpoints ($\square$)
  each. However, first finding the minimum number of discrete
  viewpoints (\CIRCLE) on the input curve $C$ and then finding the
  $m=2$ paths to visit them may give arbitrarily long solutions (at
  least one path will have to visit two
  \CIRCLE). \label{fig:badstreet}}}
\end{figure}




\section{Optimal Algorithm when Given A Set of Targets and Chain-Visible Curve}
\label{sec:dynprog}

In this section, we present an optimal algorithm for solving
Problem~\ref{prob:dynprog}. Here, we are given as input a set of
points $X$ in an $n$-sided 2D polygon that must be visually
monitored. The goal is to find a set of viewpoints $V$ along with $m$
watchmen paths that visit $V$. The watchmen paths are restricted to an
input curve $C$ that along with $X$ satisfies the chain-visibility
property (Definition~\ref{def:chainvisibility}). We
show how to compute $V$ and find the $m$ watchmen paths optimally.

In general, the viewpoints $V$ can be anywhere along $C$, i.e., no
finite candidate set for $V$ is given. We will first establish that
there always exists an optimal solution in which $V$ is restricted to
either endpoint of the intersection of $\VP(x)$ with $C$, where $x$ is
some point in $X$. Let $C_x$ denote the segment $\VP(x)\cap C$ for
$x\in X$. For ease of notation, we will assign an ordering of points
on the curve $C$. This allows us to define the left and right
endpoints for $C_x$ (equivalently, first and last points of $C_x$
along $C$). We have the following result.

\begin{lemma}
There exists an optimal solution for Problem~\ref{prob:dynprog} with
viewpoints $V^*$ such that for any $v\in V^*$, if $v$ is the left
(respectively, right) endpoint of some path $\Pi_i^*$, then $v$ must
be the right (respectively, left) endpoint of some $C_x$.
\label{lem:dynprogorder}
\end{lemma}
The proof is given in the appendix.

This lemma allows us to restrict our attention only to the set of
finite (at most $2|X|$) points on $C$. Furthermore, we need to
consider only the right endpoints of all $C_x$ for starting a path,
and only the left endpoints for ending a path. We will use dynamic
programming to find the optimal starting and ending points of $m$
paths. Before we describe the dynamic programming solution, we present
a subroutine that is useful in computing the cost of a path when the
first and last viewpoint on the path is given.

\begin{algorithm}
\DontPrintSemicolon
\KwIn{$i,j$: first and last viewpoints for a path $\Pi_i$ on $C$}
\KwIn{$X'$: set of target points such that $\forall x\in X', \VP(x)\cap
  \Pi_i\neq\emptyset$}
\KwOut{$V_i$: optimal set of viewpoints on $\Pi_i$ to cover $X'$
  (including $i$ and $j$)}
\KwOut{$J_i$: optimal cost for $\Pi_i$ to cover $X'$}
Mark all points in $X'$ as uncovered\;
Mark all points in $X'$ visible from either $i$ or $j$ as covered\;
$V_i\leftarrow\{i,j\}$\;
$p\leftarrow i$\;
\While{$\exists$ an uncovered point in $X'$}{
$q\leftarrow$ first point to the right of $p$ such that $q$ is the
  right endpoint of $C_x$ for some uncovered $x\in X'$\; \label{line:q} 
$V_i\leftarrow V_i\cup\{q\}$\;
Mark all $x\in X'$ visible from $q$ as covered\;
$p\leftarrow q$\;
}
$J_i = l(\Pi_i) + |V_i|t_m$\;
return $V_i$ and $J_i$
\caption{\textsc{OptimalSinglePath}\label{algo:optsinglepath}}
\end{algorithm}

The subroutine given in Algorithm~\ref{algo:optsinglepath} takes as
input a path $\Pi_i$ defined by its first and last viewpoint on
$C$. It also takes as input a set of target points $X'$ that are
visible from at least one point along $\Pi_i$. The output of the
subroutine is the optimal set of viewpoints $V_i$ (subject to the
condition that first and last point of $\Pi_i$ are included) and the
optimal cost $J_i$ of this path. The following lemma proves the
correctness of this algorithm.

\begin{lemma}
Let $X'$ be a set of target points and $C$ be a chain-visible
curve. Let $\Pi_i$ be some path along $C$ and $X'\subseteq X$ be
target points visible from $\Pi_i$. If the first and last viewpoints
of $\Pi_i$ are $i$ and $j$ respectively, then
Algorithm~\ref{algo:optsinglepath} computes the optimal set of
viewpoints and the optimal cost for $\Pi_i$.
\label{lem:optsinglepath}
\end{lemma}
The proof of correctness in given in the appendix.

From Lemma~\ref{lem:dynprogorder} we know that all paths in an optimal
solution start and end at the right and left endpoints of
$C_x$. Denote the set of all right and left endpoints by $R$ and $L$
respectively. We build a table of size $|L|\times|R|\times m$. The
entry $T(i,j,k)$ gives the maximum cost of the first $k$ paths, with
the $k^{th}$ path starting at some $i\in R$ and ending at some $j\in
L$, and all $k'<k$ paths ending before $i$. To correctly fill in the
entry $T(i,j,k)$, we must ensure that there does not exist any $C_x$
that starts after the $(k-1)^{th}$ path and ends before $k^{th}$ path
(Figure~\ref{fig:dynprog}). Let $I(j',i)$ be a binary indicator which
is $1$ if there exists a $C_x$ which is strictly contained between
points $j'$ and $i$, $i>j'$ (but does not contain $i$ and $j'$). Let
$\texttt{Inf}$ be a very large number. We compute $T(i,j,k)$ as
follows.

\begin{figure}[htb]
\centering{ 
\includegraphics[width=0.8\columnwidth]{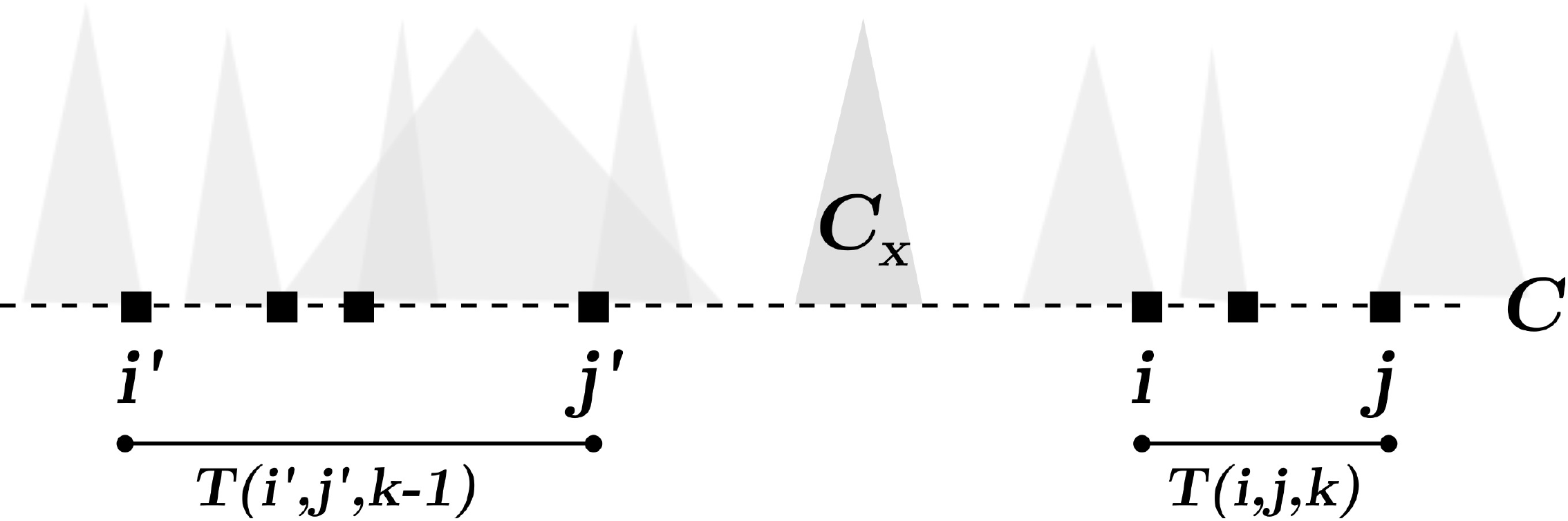}
\caption{Dynamic Programming. We search for the start and end points
  for $m$ paths. Table entry $T(i,j,k)$ gives the cost of the $k^{th}$
  part starting at $i$ and ending in $j$. We must make sure that there
  is no point $x$ such that the $C_x$ lies completely between the
  $(k-1)^{th}$ and $k^{th}$ paths. \label{fig:dynprog}}}
\end{figure}

Let $s$ be the first point on $C$. In initializing the entry
$T(i,j,k=1)$ we must ensure that there is no target such that $C_x$
ends before $i$. Thus, for all $i=1$ to $|R|$ and $j=1$ to $|L|$ we
initialize,
\begin{align*}
T(i,j,1) = 
\begin{cases}
\textsc{OptimalSinglePath}(i,j) \hfill &I(s,i) = 0\\
\texttt{Inf}\hfill &I(s,i)>0
\end{cases}
\end{align*}

To fill the rest of the entries, we first find points $j'$ and $i'$
given $i,j,k$,
\begin{align}
[j',i'] = \argmin_{j'<i,i'\leq j'}\{T(i',j',k-1)+\texttt{Inf}\cdot
I(j',i)\}.
\label{eqn:jiprimes}
\end{align}
The term $I(j',i)$ ensures that there is no $C_x$ that starts after
$j'$ but ends before $i$. Furthermore, since the $(k-1)^{th}$ path
ends at $j'<i$, we know all $C_x$ that start before $j'$ will be
covered. This only leaves two types of points to consider: (i) $C_x$
starts after $j'$ but does not end before $i$, and (ii) $C_x$ starts
after $i$. While filing in $T(i,j,k)$ we first compute $j',i'$
according to Equation~\ref{eqn:jiprimes}. Then, we verify if there
exists any point $x$ belonging to either of the two types listed
above. If not, then all points in $X$ have already been covered by the
first $k-1$ paths. Hence, we set $T(i,j,k)=T(i',j',k-1)$. If there is
exists a point belonging either of the two types listed, then
\begin{align*}
T(i,j,k)=\max\{\textsc{OptimalSinglePath}(i,j), T(i',j',k-1) \}.
\end{align*}

Additionally, if $k=m$ we must check if there is any point which has
not been covered. Let $t$ be the rightmost point of the curve $C$. If
$I(j,t)=1$, we set $T(i,j,m)$ to $+\texttt{Inf}$.

To recover the final solution, we have to find the entry $T(i,j,m)$
with the least cost. Using additional book-keeping pointers, we can
recover the optimal solution by standard dynamic programming
backtracking. The following theorem summarizes our main result for
this section.

\begin{theorem}
There exists a polynomial time algorithm that finds the optimal
solution for Problem~\ref{prob:dynprog}.
\end{theorem}

The property in Lemma~\ref{lem:dynprogorder} allowed us search over a
finite set of points for computing the endpoints of the paths. This
comes from the finiteness of the set $X$. Now, consider the case when
\emph{all} points in a polygon are to be monitored by the robots. We
may have possibly infinite candidate endpoints for the $m$ optimal
paths along $C$. Nevertheless, in the next section we will show how to
compute an approximation for the optimal paths in finite time.

\subsection{Simulations}
\label{sec:sims}
\begin{figure*}[htb]
\centering{ 
\subfigure[\label{fig:sims_env}]{\includegraphics[width=0.3\textwidth]{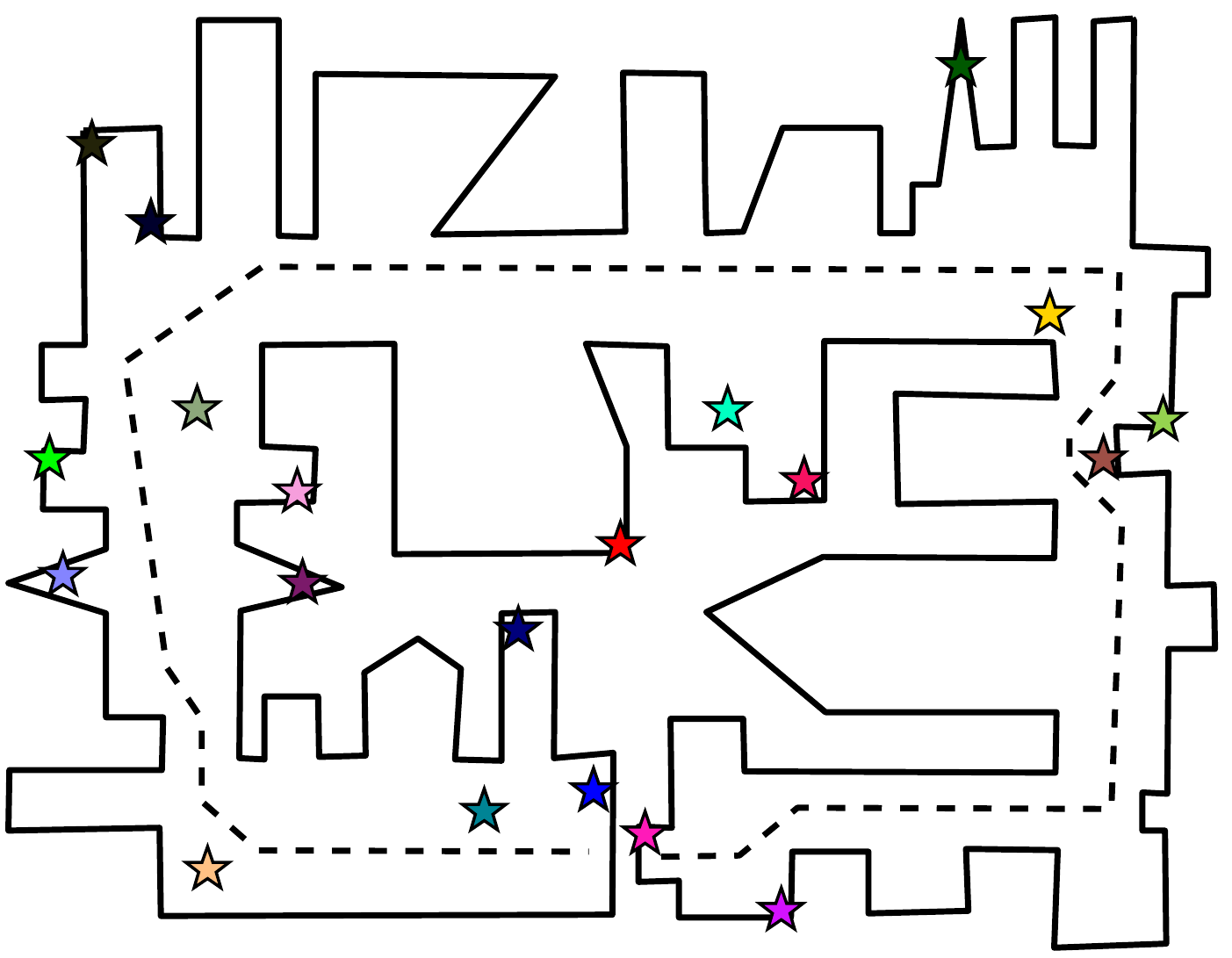}}
\subfigure[\label{fig:simnumrobots}]{\includegraphics[width=0.3\textwidth]{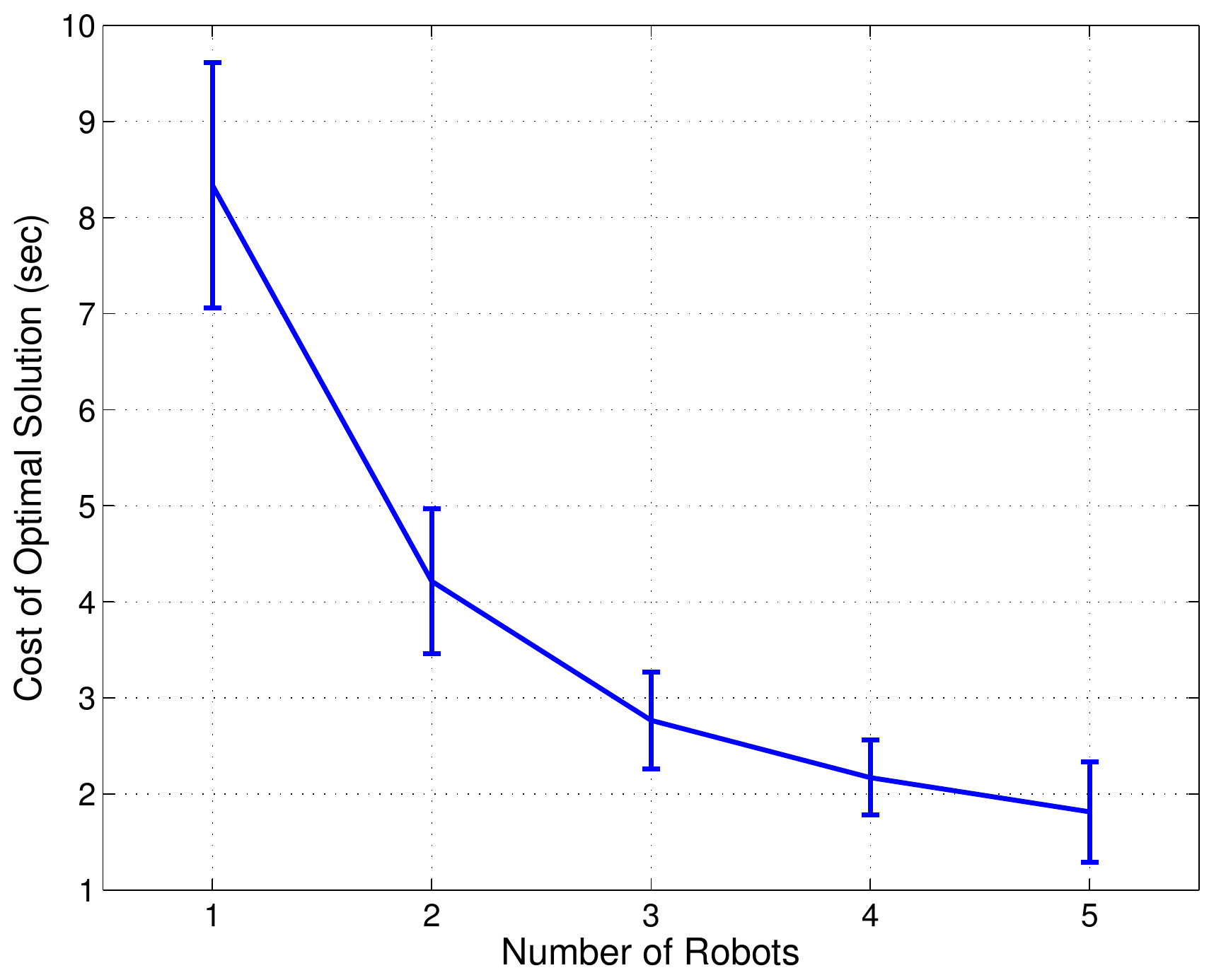}}
\subfigure[\label{fig:simnumtargets}]{\includegraphics[width=0.3\textwidth]{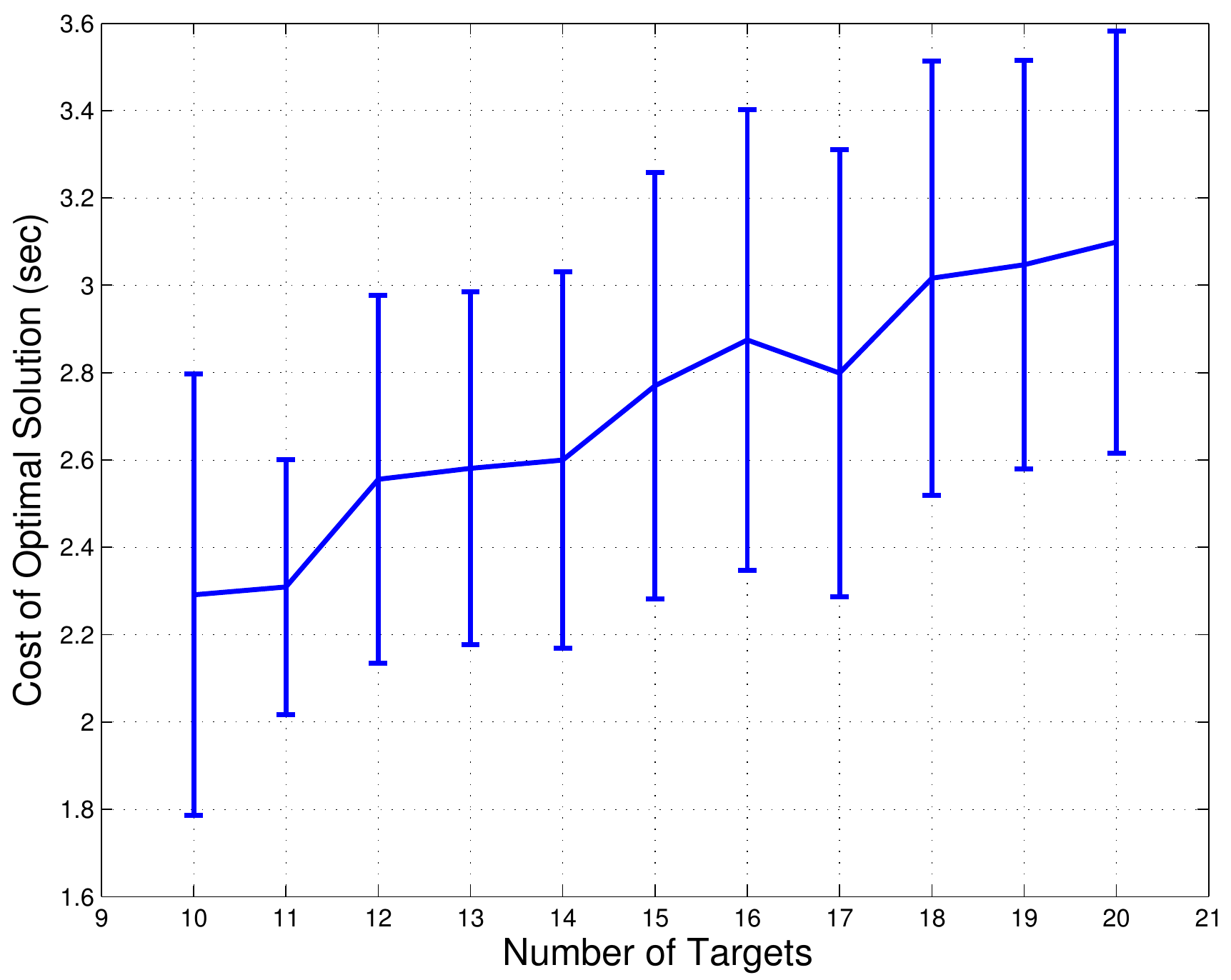}}
\caption{(a) Representative simulation instance. The target points ($\star$) are generated randomly for each trial within the polygonal environment. (b) Varying the number of robots on the optimal solution cost (makespan). The number of targets is fixed to 15 for all trials. (c) Varying the number of targets on the optimal solution cost (makespan). The number of robots is fixed to 3 for all trials.\label{fig:sims}}}
\end{figure*}

The algorithm was implemented in MATLAB using the VisiLibity~\cite{visilibity08} library for floating-point visibility computations. The polygonal environment shown in Figure~\ref{fig:sims_env} was used to generate all the simulation instances. 

\begin{figure}[htb]
\centering{ 
\includegraphics[width=0.45\columnwidth]{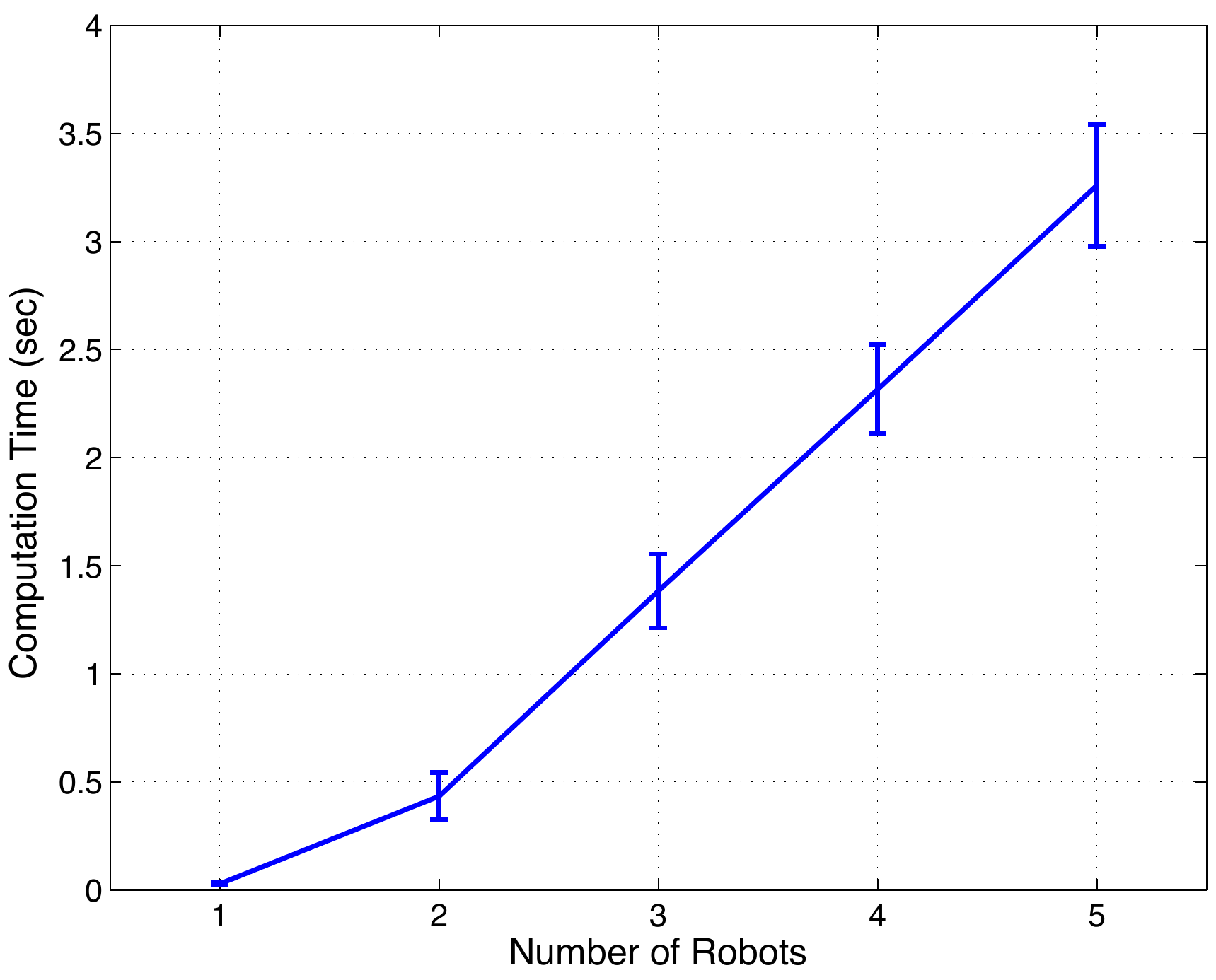}
\includegraphics[width=0.45\columnwidth]{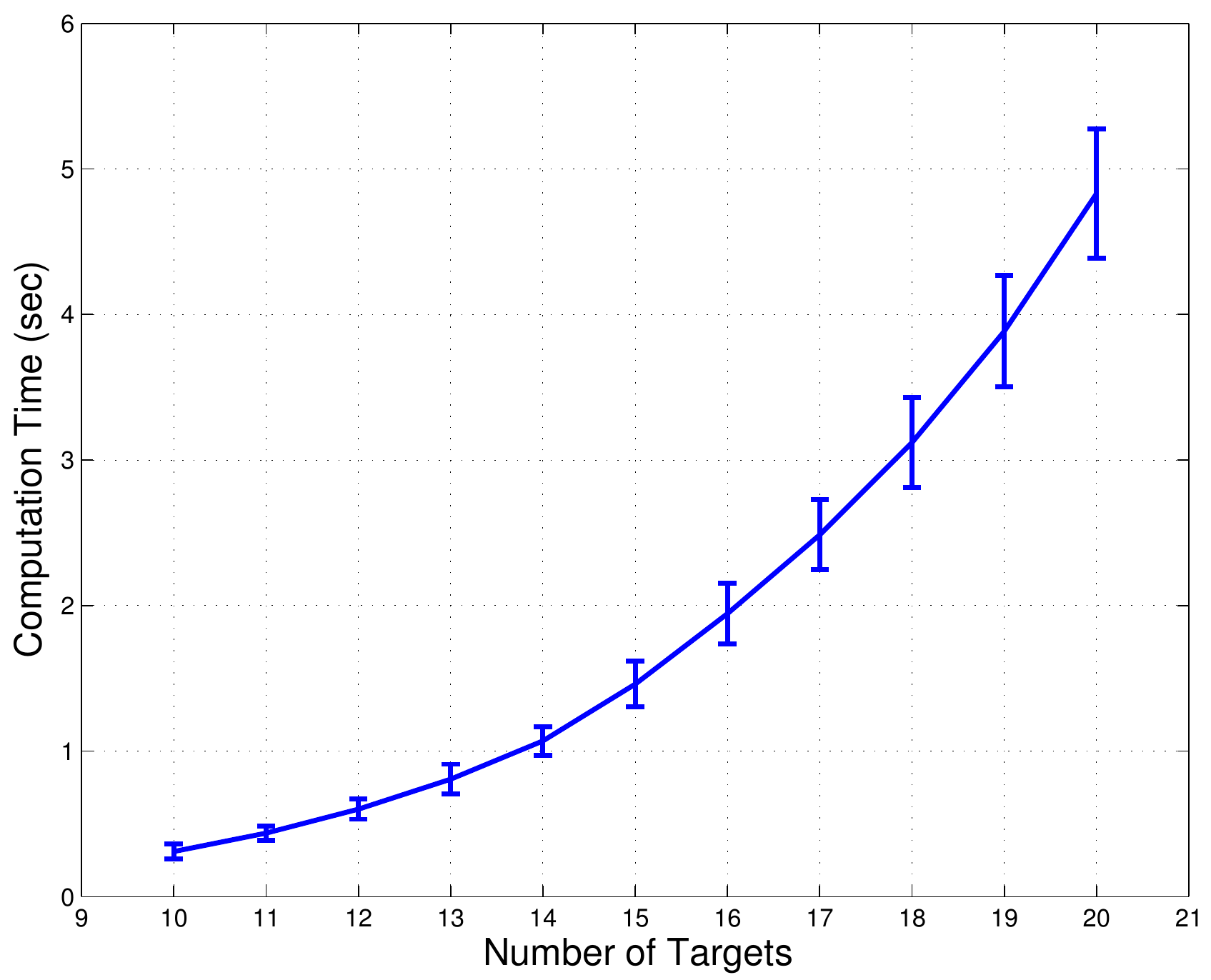}
\caption{Computation time as a function of the number of robots with 15 targets (left), and as a function of the number of targets with 3 robots (right). \label{fig:simtime}}}
\end{figure}

The simulation results are presented in Figure~\ref{fig:sims}. The plots show mean and standard deviation of the costs for 50 trials. The positions of the targets are randomly generated for each trial. All target locations are generated such that they satisfy the chain-visibility property with respect to the dashed path shown in Figure~\ref{fig:sims_env}. The measurement time $t_m$ was set to $1$\,s for all trials.

Figure~\ref{fig:simnumrobots} shows the effect of varying the number of robots on the optimal cost (makespan). 15 target locations are randomly generated for each trial. Figure~\ref{fig:simnumtargets} shows the effect of varying the number of targets. The number of robots were fixed to 3. Figure~\ref{fig:simtime} shows the computation time required for running the dynamic programming, as a function of the number of robots and the targets. 

\section{$4$--Approximation For Street Polygons}
\label{sec:street}

In this section, we present a $4$--approximation algorithm for
Problem~\ref{prob:street}. The input to the problem is a street
polygon $P$ (see Figure~\ref{fig:street} for an example) and a curve
$C$ that satisfies the chain-visibility property for all points in
$P$. Carlsson and Nilsson~\cite{carlsson1999computing} showed that
there always exists such a curve for street polygons, known as the
collapsed watchman route. They also presented an algorithm to compute
the smallest set of discrete viewpoints along such a curve to see
every point in $P$. As shown in Figure~\ref{fig:badstreet}
constructing paths directly from the optimal set of discrete
viewpoints can lead to arbitrarily worse solutions for
Problem~\ref{prob:street}. Nevertheless, in this section, we will
present an algorithm that yields a $4$--approximation starting with
the smallest set of discrete viewpoints.

Let $g_s$ and $g_t$ be the first and last points of the input curve
$C$. Let $p$ and $q$ by any points on $C$ ($p$ to the left of
$q$). Let $C[p,q]$ denote the set of all points on $C$ between $p$ and
$q$.  We use the following definition of a \emph{limit point} of a
point $p$ adapted from~\cite{carlsson1999computing}.
\begin{definition}
The limit point of a point $p$ on $C$, denoted by $lp(p)$, is defined
as the first point on $C$ to the right of $p$ such that $lp(p)$ is the
right endpoint of $C_x$ for any $x\in
closure(P\setminus\VP(C[g_s,p]))$.\footnote{The closure of a set of
  points is the union of the set of points with its boundary.}
\end{definition}
In other words, $lp(p)$ is the right endpoint of a $C_x$ closest to
$p$ and to its right, such that $x$ is not visible from any point to
the left of $p$, including $p$. Figure~\ref{fig:limitpoint} shows an
example.

\begin{figure}[htb]
\centering{ 
\includegraphics[width=0.8\columnwidth]{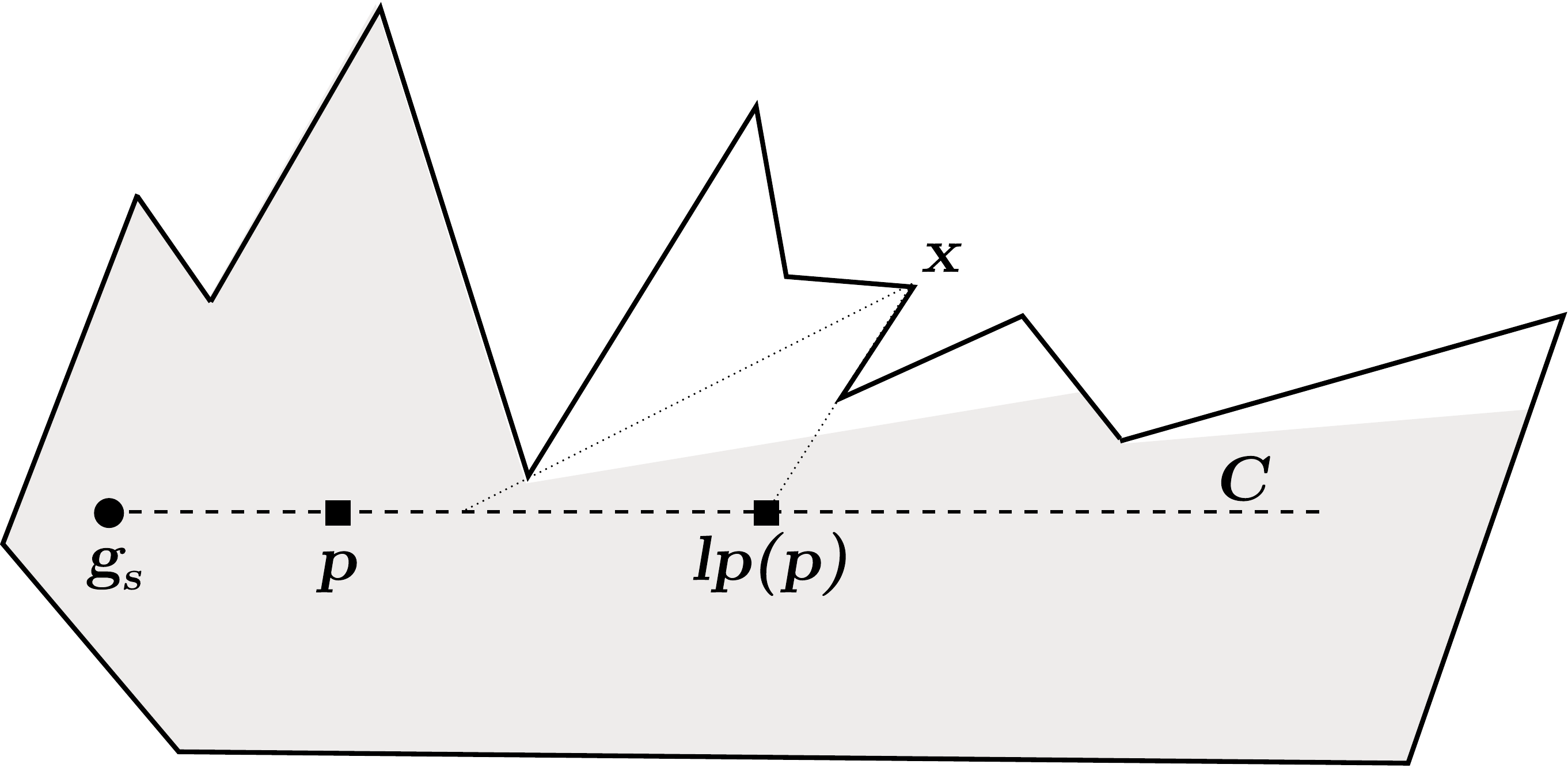}
\caption{Limit point $lp(p)$ is the leftmost point to the right of $p$
  which is also the right endpoint of $C_x$ for some $x$ not in
  $\VP(C[g_s,p])$.\label{fig:limitpoint}}}
\end{figure}

In the previous section, we implicitly used the concept of a limit
point in Line~\ref{line:q} of Algorithm~\ref{algo:optsinglepath}. The
limit point of any point $p$ given a curve $C$ can be computed
efficiently in polynomial time~\cite{carlsson1999computing}, which we
will use in our algorithm. For the first limit point, we have the
following result.
\begin{lemma}[\cite{carlsson1999computing}]
If $g_1$ is the first point on $C$ to the right $g_s$ such that $g_1$
is the right endpoint of $C_x$ for any $x\in P$, then $\VP(C[g_s,g_1])
= \VP(g_1).$
\label{lem:g1}
\end{lemma}
Carlsson and Nilsson~\cite{carlsson1999computing} also presented an
algorithm to compute the first viewpoint $g_1$ satisfying the above
condition. 

It is easy to see that any solution where the first path starts before
$g_1$ can be converted to another valid solution of equal or less cost
where the first path starts at $g_1$. The subroutine given in
Algorithm~\ref{algo:streetsub} starts with $g_1$ to compute $m$
paths. Let $T^*$ be the cost of the optimal solution for some instance
of Problem~\ref{prob:street}. The subroutine takes as input a guess
for $T^*$, say $\hat{T}$. We first compute the smallest set of
viewpoints $G^*$ sufficient to see every point in the environment
using~\cite{carlsson1999computing}. The rest of the algorithm
constructs $m$ paths such that the cost of each path is at most
$4\hat{T}$.

Lemma~\ref{lem:guess} shows if the set of paths computed by the
algorithm does not see every point in $P$, then our guess for $T^*$ is
too small. Thus, we can start with a small initial guess for $T^*$,
say $\hat{T}=t_m$, and use binary search to determine the optimal
value $T^*$. Since all tours computed cost less than $4\hat{T}$, we
obtain a $4$--approximation when our guess $\hat{T}=T^*$.

\begin{algorithm}
\DontPrintSemicolon
\KwIn{$P$: a street polygon}
\KwIn{$C$: collapsed watchman route from $g_s$ to $g_t$}
\KwIn{$t_m$: measurement cost per viewpoint}
\KwIn{$\hat{T}$: guess for the cost of optimal solution}
$G^*\leftarrow\{g_i\}$: smallest set of viewpoints on $C$ (\cite{carlsson1999computing})\;
$l\leftarrow g_1$\;
\For{$i=1$ to $m$}{
$r\leftarrow$ point along $C$ length $\hat{T}$ away from $l$ (set to
  $g_t$ if no such point exists)\;
$V_i'\leftarrow\{g_i|g_i\in G^*, l<g_i<r\}$\;
\If{$|V_i'|> \dfrac{\hat{T}}{t_m}$}{
$V_i'\leftarrow$ first $\lceil\hat{T}/t_m\rceil$ viewpoints in
  $V_i'$\;
$r\leftarrow$ last point in $V_i'$\;
}
$\Pi_i\leftarrow$ path starting at $l$ and ending at $r$\;
$V_i=\{l\}\cup V_i'\cup\{r\}$\tcp*{viewpoints on $\Pi_i$}
$l\leftarrow lp(r)$\;
}
return \textsc{Success} if $\cup\VP(V_i)=P$, \textsc{Failure} otherwise.
\caption{\textsc{StreetSubroutine}\label{algo:streetsub}}
\end{algorithm}

We first prove that the discrete set of viewpoints computed by
Algorithm~\ref{algo:streetsub} are correct.
\begin{lemma}
Let $r$ be the last point on the last path and $V=\cup V_i$ be the set
of all viewpoints over all paths given by
Algorithm~\ref{algo:streetsub}. If a point $x\in P$ is visible from
$C[g_s,r]$, then $x$ is also visible from the discrete set of
viewpoints $V$.
\label{lem:correctsub}
\end{lemma}
\begin{proof}
Suppose there is a point $x$ visible from $C[g_s,r]$ which is not
visible from any point in $V$. Let $x_l$ and $x_r$ be the left and
right endpoints of $C_x$. $x_l$ must lie to the left of $r$ since $x$
is visible from $C[g_s,r]$. Similarly, $x_r$ must lie to the left of
$r$ otherwise since $r\in V$, $x$ will be covered by $r$. From
Lemma~\ref{lem:g1} and the fact that $g_1\in V$, $x_l$ must lie to the
right of $g_1$. Thus we have $g_1<x_l\leq x_r < r$. We omit the rest
of the proof since it is similar to that of
Lemma~\ref{lem:optsinglepath}.


\end{proof}

\begin{lemma}
Let $\Pi$ be the set of paths computed using
Algorithm~\ref{algo:streetsub} and $r$ be the last endpoint of the
rightmost path. Let $\Pi'$ be any set of $m$ paths with $r'$ the last
endpoint of the rightmost path, such that all points visible from
$C[g_s,r']$ are covered by $\Pi'$. If the cost of $\Pi'$ is at most
$\hat{T}$, then $r'$ cannot be to the right of $r$.
\label{lem:right}
\end{lemma}
\begin{proof}
From Lemma~\ref{lem:g1} and the definition of limit point, we can say
that the first path in $\Pi'$ starts from $g_1$. We will prove the
lemma by induction on the index of the path. Specifically, we will
show that the right endpoint of the $i^{th}$ path in $\Pi$, say
$\Pi_i$, cannot be to the left of the right endpoint of the $i^{th}$
path in $\Pi'$, say $\Pi_i'$. For ease of notation, we will refer to
the left endpoint of the $i^{th}$ path by $l_i$ and correspondingly
$r_i$. 

\textbf{Base case.} We have two possibilities: (i) $r_1$ is $\hat{T}$
away from $g_1$, (ii) $\Pi_1$ contains at least
$\lceil\hat{T}/t_m\rceil$ viewpoints from $G^*$. For (i), since the
cost of $\Pi_1'$ is at most $\hat{T}$, its length cannot be greater
than $\hat{T}$. Hence, $r_1'$ cannot be to the right of $r_1$. For
(ii), suppose $r_1'$ is to the right of $r_1$. Then $\Pi_i'$ must
contain at least $\lceil\hat{T}/t_m\rceil$ viewpoints from the
optimality of $G^*$. Thus, the cost of $\Pi_1'$ is greater than
$\hat{T}$ which is a contradiction. The base of the induction holds.

\textbf{Inductive step.} Suppose that $r_{i-1}'$ is to the left or
coincident with $r_{i-1}$. We claim that $l_{i-1}'$ must be to the
left or coincident with $l_i$. Suppose not. By construction,
$l_i=lp(r_{i-1})$. Let $x$ be a point such that $l_i$ is the right
endpoint of $C_x$ and $x$ is not visible from $C[g_s,r_{i-1}]$. Such
an $x$ always exists according to the definition of a limit
point. Hence, $C_x$ lies completely between $r_{i-1}'$ and $l_i'$
implying $x$ is not covered by $\Pi'$, which is a
contradiction. Hence, $l_{i}'$ is to the left or coincident with
$l_i$. Now, using an argument similar to the base case, we can show
that $r_i'$ cannot be to the right of $r_i$ proving the induction.

Hence, for all $i$, $r_i'$ cannot be to the right of $r_i$ proving the
lemma.

\end{proof}

\begin{lemma}
If Algorithm~\ref{algo:streetsub} returns \textsc{Failure} then
$\hat{T} < T^*$.
\label{lem:guess}
\end{lemma}
\begin{proof}
Suppose $T^*\leq \hat{T}$. Let $\Pi^*$ be the optimal set of $m$
paths. Let $r$ and $r^*$ be the rightmost points on $\Pi$ and $\Pi^*$
respectively. From Lemma~\ref{lem:right}, we know $r^*$ cannot be to
the right of $r$.

Algorithm~\ref{algo:streetsub} returns \textsc{Failure} when there
exists some point, say $x\in P$, not covered by $\Pi$. From
Lemma~\ref{lem:correctsub}, $x$ must not be visible from
$C[g_s,r]$. Thus, $x$ cannot be visible from $C[g_s,r^*]$. Hence, the
optimal set of paths do not see every point in $P$ which is a
contradiction. 
\end{proof}

We now bound the cost of the solution produced by our subroutine.
\begin{lemma}
The maximum cost of any path produced by
Algorithm~\ref{algo:streetsub} is at most $4\hat{T}$.
\end{lemma}
\begin{proof}
By construction, the length of any path is at most $\hat{T}$. The
number of measurements are at most $\lceil\hat{T}/t_m\rceil + 2\leq
3\hat{T}/t_m$. Hence, the cost of any path is at most $4\hat{T}$.
\end{proof}

Combining all the above lemmas, we can state our main result for this
section.
\begin{theorem}
There exists a $4$--approximation for Problem~\ref{prob:street}.
\end{theorem}

The minimum number of discrete viewpoints $|G^*|$ can be significantly
larger than the number of vertices of the polygon. $G^*$ can be
computed in time polynomial in the input size (i.e., the number of
sides in the input polygon) and the output size (i.e., $|G^*|$).
Consequently, the running time for each invocation of the subroutine
in Algorithm~\ref{algo:streetsub} is also polynomial in the input and
output size. The optimal value $T^*$ can be computed in $\bigO(\log
T^*)$ invocations of the subroutine via binary search. We can reduce
the overall running time by terminating the binary search early, at
the expense of the approximation factor. 

\section{General Solution}	\label{sec:general}
In this section, we address the general case for the multi-robot watchmen route problem. We remove the restriction of street polygons and requiring a chain-visible curve as input. However, the added generality comes at the expense of some relaxations. We assume that a finite set of candidate measurement locations, $V$, is given as input. The goal is to find tours for each robot visiting a subset of $V$ such that they collectively see all the targets. Note that there is no further assumption (\eg, chain-visibility) on $V$. Consequently, $V$ can be computed by simply discretizing the environment either uniformly or using the strategy in Deshpande et al.~\cite{deshpande2007pseudopolynomial}. Instead of minimizing the maximum times of the tour, we resort to minimizing the sum of the length of all tours (\ie, $t_m=0$). 

Our contribution is to show how to solve this general version of the multi-robot watchman route problem by reducing it to a TSP instance. The resulting TSP instance is not necessarily metric, and consequently existing polynomial time approximation algorithms cannot be directly applied. Instead, we directly find the optimal TSP solution leveraging sophisticated TSP solvers (\eg, concorde~\cite{applegate2006concorde}). This is motivated by recent work by Mathew et al.~\cite{mathew2015multirobot} who used a similar approach for solving a multi-robot rendezvous problem. We demonstrate that this algorithm finds the optimal solution faster than directly solving the original problem using an Integer Linear Programming solver.

In the following we describe our reduction and prove its correctness. We also present empirical results from simulations for a 2D case. However, the following algorithm also works for 3D problems and can incorporate additional sensing range and motion constraints.

\subsection{Reduction to TSP}

Let $P$ denote the environment in which the targets points $X$ are located. We are given a set of candidate viewpoints $V$ within $P$. We denote the $j^{th}$ viewpoint by $v_j\in V$ where $j=\{1,...,n\}$. For each target $x_i\in X$, we create cluster of viewpoints, $C_i=\{v_j\mid v_i\subseteq V \text{ can see } x_i\}$. Without loss of generality, we assume that $V$ is such that each target $x_i$ is seen from at least one viewpoint (\ie, $|C_i| \geq 1$ for all $i$). One way of generating a valid set $V$ is by sampling or discretizing the visibility polygons for all $x_i$. For the special case, when $X$ is the set of all points in a 2D polygon, we can generate $V$ by imposing a grid inside $P$ using the strategy from~\cite{deshpande2007pseudopolynomial}.

If a robot visits any viewpoint in cluster $C_i$, then it ensures that target $x_i$ is seen by the corresponding robot. Therefore, the goal is to find a set of tours, one per robot, such that at least one viewpoint in each $C_i$ is visited. Note that the clusters need not be disjoint. This problem is equivalent to the multi-robot Generalized Traveling Salesman Problem (GTSP)~\cite{Lien1993}.  

The input to GTSP is a set of clusters, each containing one or more nodes from a connected graph. The goal in single robot GTSP is to find the minimum length tour that visits at least one node in each cluster.\footnote{There are versions of GTSP with additional restrictions of visiting exactly one node in each cluster or visiting each cluster exactly once. We consider the less restrictive version where the robot is allowed to visit a cluster multiple times while still ensuring a specific node is visited no more than once.} GTSP is NP-hard~\cite{Lien1993} since it generalizes TSP.  

Noon and Bean~\cite{noon1993efficient} and Lien and Ma~\cite{Lien1993} presented two polynomial time reductions of GTSP into a TSP instance such that the optimal TSP tour yields the optimal GTSP solution. The resulting TSP instance is not necessarily metric and consequently the standard approximation algorithms (\eg,~\cite{arora1996polynomial}) cannot be applied. We can find the optimal solution of the TSP instance directly using potentially exponential time algorithms. A number of sophisticated implementations have been developed for solving large TSP instances to optimality~\cite{applegate2006concorde}. In particular, we use the \emph{concorde} solver which yields the best-known solutions to large TSP instances~\cite{applegate2011traveling}. In Section~\ref{sec:compstudies}, we compare this approach with a generic Integer Linear Programming solver.

The reductions from GTSP to TSP proposed by Noon and Bean~\cite{noon1993efficient} and Lien and Ma~\cite{Lien1993} are for finding a single robot GTSP tour. In our case, we are interested in finding $m$ tours -- one for each of the $m$ robots.  Mathew et al.~\cite{mathew2015multirobot} presented a multi-robot extension for the Noon-Bean transformation. This transformation is applicable for the case when the clusters in the GTSP instance are disjoint (\ie, $|C_i \cap C_j| = 0$, for all $i\neq j$). With slight modification, we can apply the transformation to the case of possibly overlapping clusters as given below. 

We assume that the path for the $i^{th}$ robot must start at a specific vertex, $v_d^i$, and end at a specific vertex $v_f^i$. We model three scenarios:
\begin{enumerate}
\item \textsc{SameDepot}: All robots start and finish their tours at the same location. That is, $v_d^i = v_d^j = v_f^i = v_f^j$ for all $i$ and $j$.
\item \textsc{SameFinishDepot}: All robots finish their tours at the same location but may have unique starting locations. That is, $v_f^i = v_f^j$ for all $i$ and $j$.
\item \textsc{InterchangeableDepots}: There are $m$ fixed depots and initially there is one robot at each depot. The robots can end their paths at any of the $m$ depots with the restriction that each depot must have one robot at the end.
\end{enumerate}

For all the scenarios, the algorithm given below finds $m$ tours such that the sum of the lengths of all tours is minimized and each target is seen.

\begin{figure*}[tb]
\centering{
\subfigure[Input Instance\label{fig:map_noon}]{\includegraphics[width=0.3\textwidth]{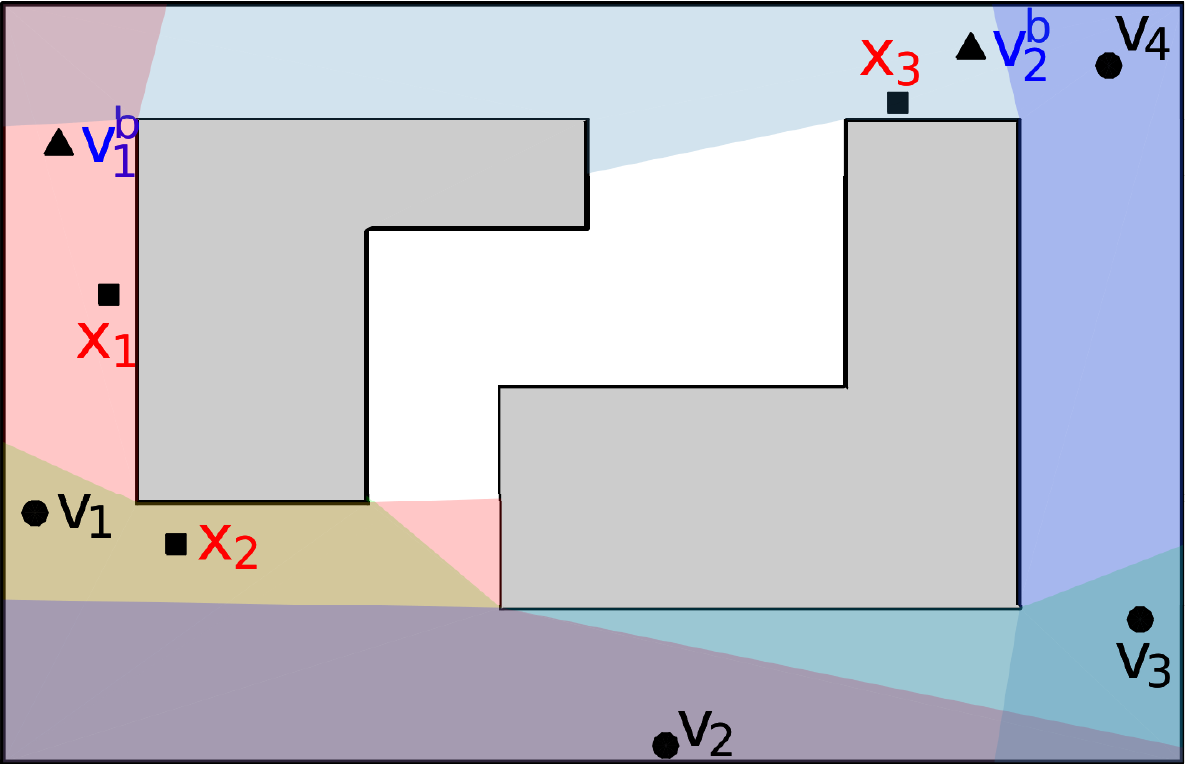}}
\subfigure[Input graph\label{fig:Given_Graph}]{\includegraphics[width=0.3\textwidth]{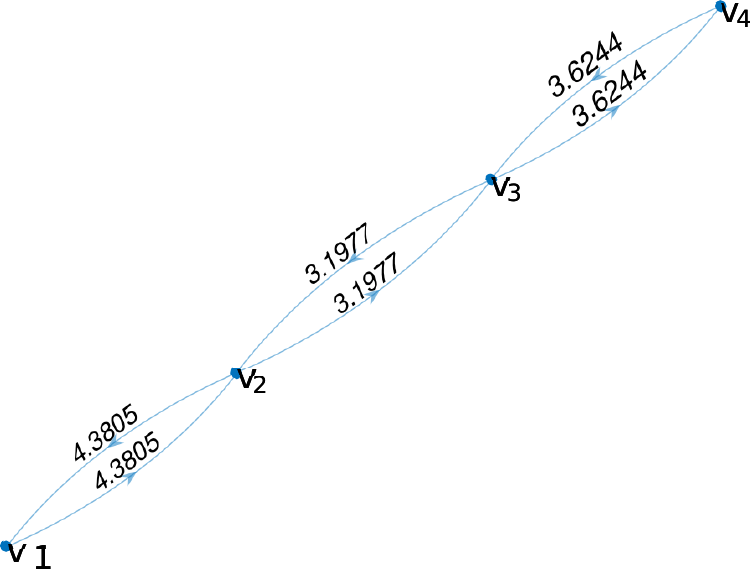}}
\subfigure[Redundant nodes removed\label{fig:comp_orphan_remove}]{\includegraphics[width=0.3\textwidth]{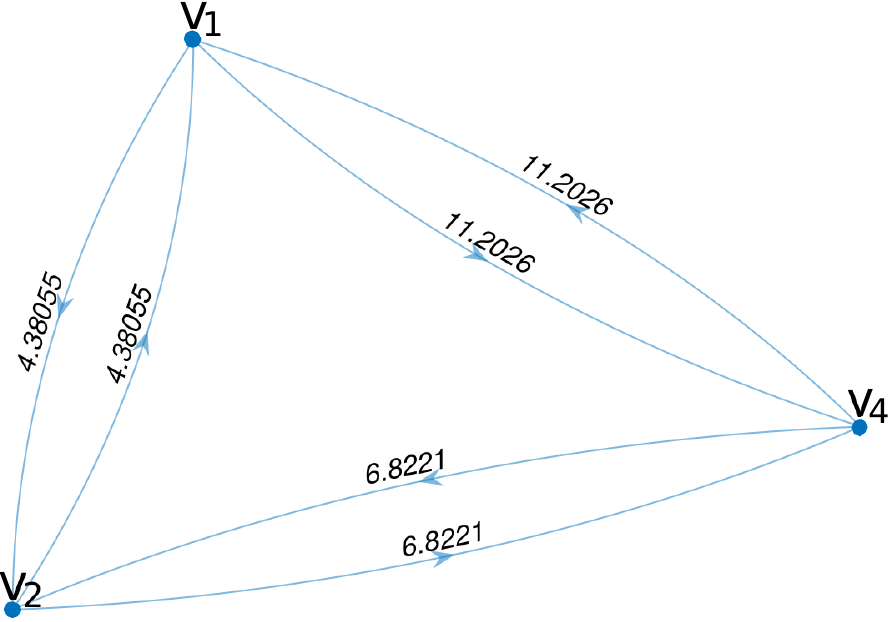}}
\subfigure[Duplicate nodes added to create non-intersecting clusters\label{fig:duplicate_node}]{\includegraphics[width=0.3\textwidth]{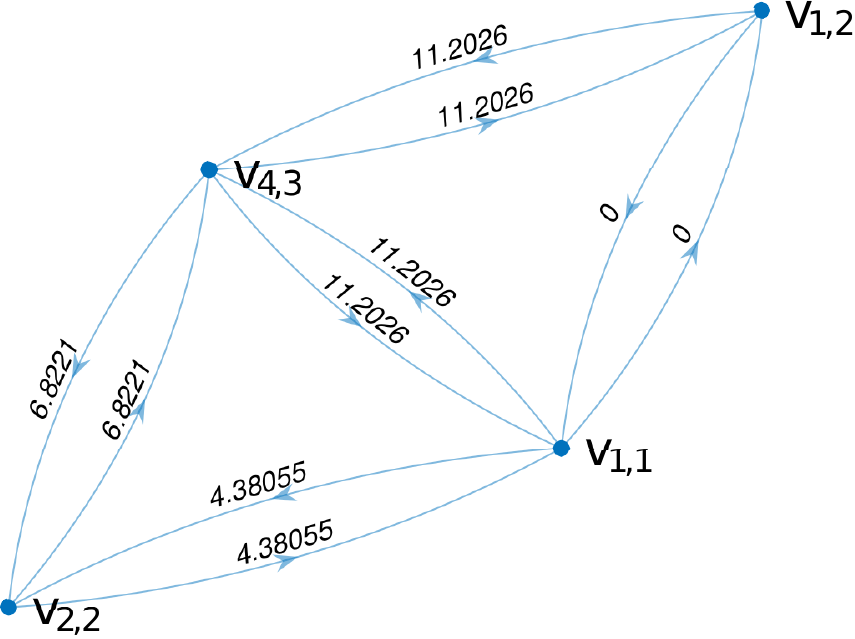}}
\subfigure[Noon-Bean transformation with penalty added and tail-shifted (red) edges.\label{fig:noonbeaned}]{\includegraphics[width=0.3\textwidth]{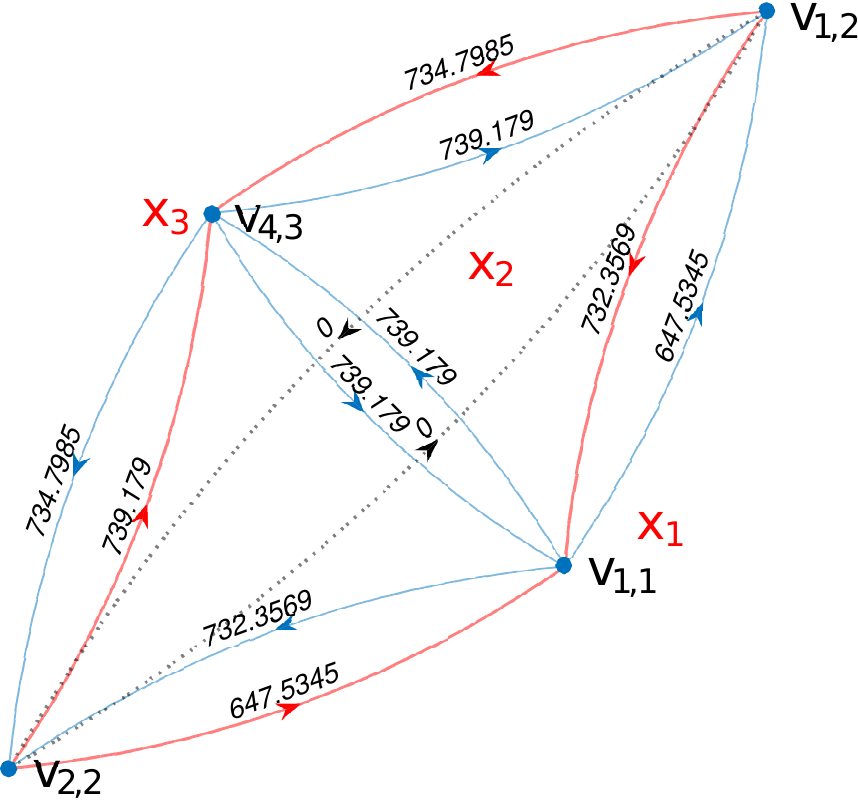}}
\subfigure[Robot depots added to the graph. Outgoing depot edges are shown in green, inter-cluster tail shifted edges are shown in red, incoming depot edges are shown in black solid, incoming tail-shifted depot edges are shown in magenta, intracluster zero cost edges are shown as dotted-black.\label{fig:G_atsp_final}]{\includegraphics[width=0.3\textwidth]{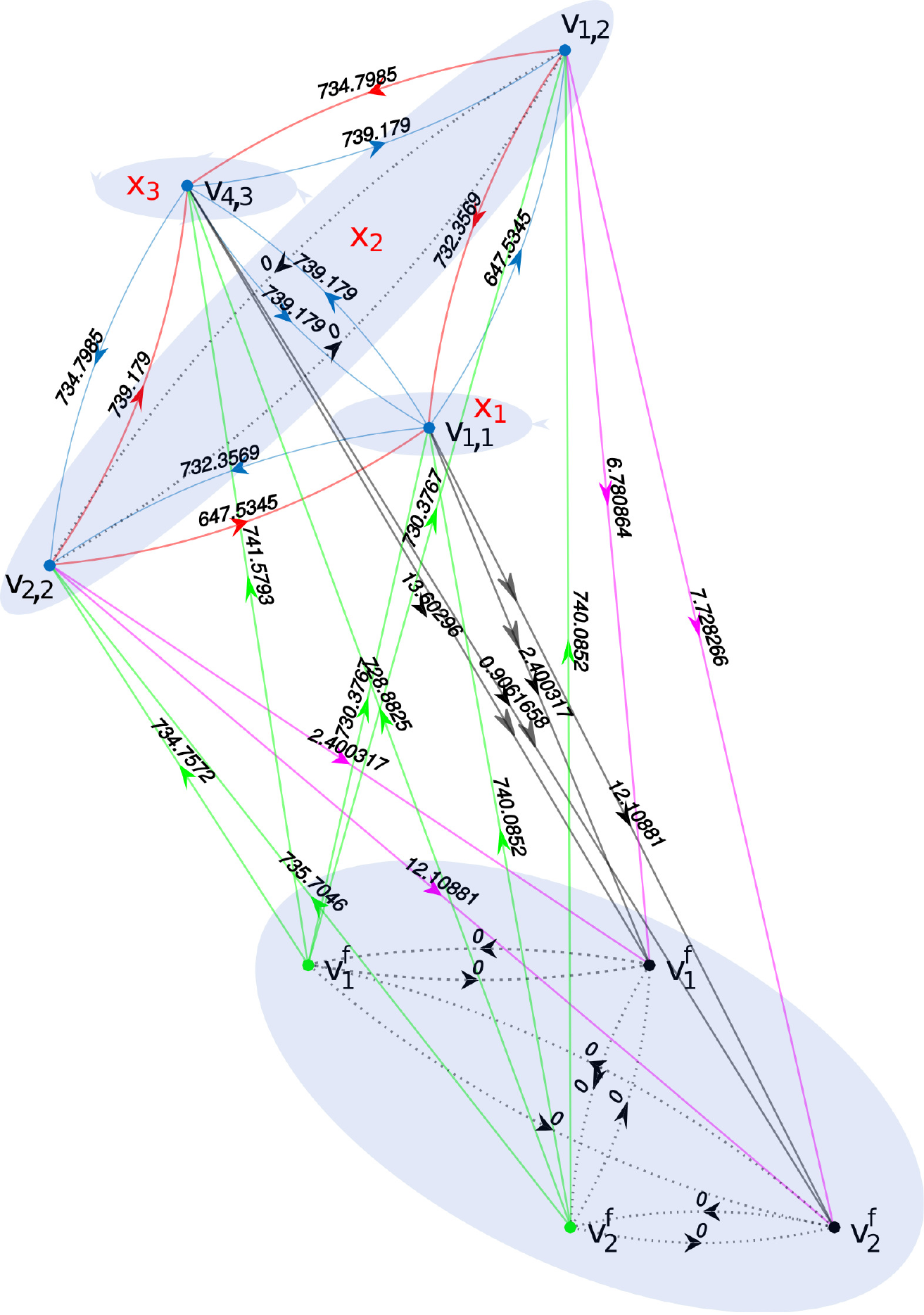}}
\caption{Transforming the multi-robot watchmen route problem into an asymmetric TSP instance.\label{fig:noonbean}}
}
\end{figure*}

The reduction consists of five main steps. First, represent the given instance as a graph. Second, form a metric completion of the graph and remove viewpoints that cannot see any target. Third, convert the overlapping clusters to non-overlapping ones. Fourth, use a modified Noon-Bean reduction~\cite{noon1993efficient} to convert the GTSP instance into a TSP instance. Fifth, add start and finish depots as nodes in the graph. The detailed description of each step is given below. 

\begin{enumerate}
\item 
(Figures~\ref{fig:map_noon}--\ref{fig:Given_Graph})
\label{stepMtoG} Represent $P$ as a graph $G^{g}=(V,E,\hat{C})$. The vertices are the viewpoints $V{\in}P$. Add an edge $(v_i,v_j)\in E$ where $v_i, v_j \in V$ and $v_i$ is visible from $v_j$. The cost of $(v_i,v_j)$ is the Euclidean distance between $v_i$ and $v_j$. Define a set of clusters $\hat{C}=\{C_i,\ldots,C_k\}$ where $C_i=\{v_j \mid v_j \in V$ is visible from target $x_i$\}.
\item 
(Figures~\ref{fig:Given_Graph} to~\ref{fig:comp_orphan_remove})
Complete the graph $G^{g}$ in Step{-}{\ref{stepMtoG}} to give $G^{c}=(V^c,E^c,\hat{C}$). Cost of an edge $(v_i,v_j)$ in $E^{c}$ is equal to the cost of shortest path between $v_i$ and $v_j$ in $G^g$. Then, remove isolated nodes $V^{r}$ that are not present in any cluster in $\hat{C}$. This in turn removes edges  $E^{r}=\{(v_i,v_j) \in E^c \mid \forall v_i \in V^r$\}. Above operations give us $G^{cr}$=($V^{cr},E^{cr},\hat{C}$) where $V^{cr}=V \setminus V^{r}$ and $E^{cr}=E^{c}\setminus E^r$. \footnote{The completion may result in some nodes to be visited multiple times in the final tour.}

\item 
Carry out the first step in Noon-Bean transformation~\cite{noon1993efficient}, the I-N transformation, as follows: $G^{cr}{=}(V^{cr},E^{cr},\hat{C})$  is converted to a graph $G^{in}=(V^{in},E^{in},\hat{C}^{in})$ that has non-intersecting set of clusters.
We go through following steps:
\begin{itemize}
\item 
Create set of nodes $V^{in}=\{v_{i,j} \mid \forall v_{i} \in V^{cr} \text{ and }\forall C_j \ni v_i, C_j \in \hat{C}\}$. 
\item Create set of clusters $\hat{C}^{in}$=$\{C^{in}_j\}$ where ${~}C^{in}_j=\{v_{i,j} \mid \forall i \in \{1,\ldots,n\}, v_{i,j} \in V^{in}\}$.
\item
For all nodes of the form $\{v_{i,j}, v_{i,k}\}{\in}V^{in}$ where $j{\neq}k$, create an \lq{intranode-intercluster}\rq { }edge (${v_{i,j}, v_{i,k}}$) ${\in}E^{in}$ and assign a zero cost to this edge.
\item
For all nodes of the form $\{v_{i,p}, v_{j,q}\}{\in}V^{in}$ create an edge \{($v_{i,p}, v_{j,q}$) ${\in}E^{in}$ where $p{\neq}q$ and $i{\neq}j$\}, and assign the cost of this edge equal to the cost of corresponding edge $(v_i, v_j) \in E^{cr}$. Refer to Figures~\ref{fig:comp_orphan_remove}--\ref{fig:duplicate_node}.
\item
Choose $\alpha$ greater than the cost of any tour in $P$ that visits all targets $x_i$. Add this penalty $\alpha$ to all edges in $E^{in}$ except zero cost edges.\footnote{We use the cost of an arbitrary tour in the original GTSP for this step instead of using sum of the cost of all the edges as used in P2 of Noon-Bean
Method~\cite{noon1993efficient}. This is used to prevent numerical issues with large penalties in concorde.} The exact expression for $\alpha$ will be defined in Step~\ref{botadd}. 

\end{itemize}

\item 
Complete the Noon-Bean reduction by adding intracluster edges, tail-shifting, and imposing another penalty. Convert $G^{in}=(V^{in},E^{in},\hat{C}^{in})$ to a new graph $G^{t}=(V^{t},E^{t}, C^{t})$ as follows:

\begin{itemize}
\item
Copy the vertices, edges and clusters: $V^t=V^{in}$, $E^t=E^{in}$, and $\hat{C}^t=\hat{C}^{in}$. Create edges to connect all nodes in cluster $C^{in}_j \in \hat{C}^{in}$ by an intracluster cycle in any order. That is, $\{v_{i_1,j} \shortrightarrow v_{i_2,j} \shortrightarrow \ldots \shortrightarrow v_{i_p,j} \shortrightarrow v_{i_1,j} \mid \{v_{i_1,j},\ldots,v_{i_p,j}\} \in \hat{C}^{in}_j, \forall j \in \{1,...,k\}\}$. Here $v_{i_1,j}{\shortrightarrow}v_{i_2,j}$ represents an edge ($v_{i_1, j}{,}v_{i_2,j}$). 
To mark these edges we assign a cost of $-1$ to them. Add these edges to $E^t$.

\item
For all intercluster edges $E^{tl} \in E^t$ where $E^{tl}=\{(v_{ip}, v_{jq}) \mid p \neq q, v_{ip} \in C^t_p, v_{jq} \in C^t_q, C^t_p, C^t_q \in  \hat{C}^t\}$, we move the tail of all edges $(v_{ip}, v_{jq} )  \in E^{tl}$ to the previous node $v_{i_{{\mhyphen}1}p}$ in the intracluster cycle defined above. That is, $ (v_{ip}, v_{jq})$ changes to $(v_{i_{{\mhyphen}1}p}, v_{jq})$.

\item
We choose a cost $\beta$ to be greater than cost of any tour in the environment $P$ that visits all targets $x_i$ with a $\alpha$ penalty added to each of the edges in the tour. Add this penalty $\beta$ to all edges in $E^{t}$ except those marked with $-1$ costs. Replace the cost of all the edges with cost $-1$ to a cost of zero.

\end{itemize}

\item \label{botadd}Add robot depots to the graph. Create a new graph $G^b=\{V^b, E^b\}$. For all the depots, add $v_{i}^d$ and $v_{i}^f$ to $V^b$. For the \textsc{InterchangeableDepots} we create two copies of each depot. For \textsc{SameDepot} we create $2m$ copies of the depot, and for \textsc{SameFinishDepot} we create $m$ copies of the finish depot. Create zero cost edges $E^b$ between all pairs of depot nodes. We get the final directed TSP graph $G^f=(V^f,E^f)$. Here $V^f = V^b \cup V^t$ and $E^f = E^b \cup E^t$. We also add the following edges to $E^f$:
\begin{itemize}
\item
Create depot outgoing edges $E^{o}=\{(v_{i}^d,v_{jp}) \mid \forall v_{i}^d \in V^b,\forall  v_{jp} \in V^t\}$. Assign a cost $\gamma{_i}$ to all edges $(v_{i}^d, v_{jp})$ equal to the cost of the shortest path connecting $v_i^d$ to $v_{jp}$ in $P$ restricted to $V$. Also, add a penalty $\alpha + \beta$ to these edges. Add $E^o$ to $E^f$.
\item
Create tail shifted depot incoming edges $E^{i} =\{( v_{jp}, v_{i}^f \mid \forall v_{jp} \in V^t, \forall v_{i}^f \in V^b\}$. Also, add the cost $\gamma{_i}$ to $(v_{jp}, v_{i}^f)$ as above but without penalty. Then we move tail of all the edges $(v_{jp}, v_{i}^f) \in E^{i} $ to the previous node $v_{j_{{\mhyphen}1}p}$ in the intracluster cycle defined in Step 4 above. That is, $(v_{jp}, v_{i}^f)$ changes to $(v_{j_{{\mhyphen}1}p}, v_{i}^f)$. Add $E^i$ to $E^f$. 
\item
Define $\alpha$=2(cost of MST in $G^{cr}$)+$2m|$cost in $E^i|_{max}$. Define
 $\beta$=$2{\alpha}$($1+$($\#$edges in MST in $G^{cr}$))+$2m$($1+{\alpha}$).
\end{itemize}
\item This gives us a directed TSP input instance. This can be converted to an undirected TSP using transformation used by Karp~\cite{karp1972reducibility}.
\end{enumerate}

The correctness of the algorithm follows from the correctness of the original Noon-Bean reduction~\cite{noon1993efficient}. The major change to the reduction is the addition of the complete graph between the depot vertices, $V_b$. The only incoming edges to the start depots, $v_i^d$, are from the finish depot vertices, $v_j^d$. Similarly, the only outgoing edges from the finish depot vertices are to the start depot vertices. Consequently, whenever the optimal TSP tour visits a finish depot vertex it must take a zero cost edge to a start depot vertex, from which it may either to a node in $V^t$ or to another finish depot vertex. Therefore, the TSP tour visits an alternating sequence of start and finish depot vertices with possibly non-zero viewpoints (\ie, $V^t$ vertices) in between. We can therefore partition the TSP tour into $m$ subtours from start depot to finish depots. This gives us paths for the $m$ robots. One or more of these subtours may be empty, in which case the optimal solution uses fewer than $m$ robots. This can happen since the algorithm minimizes the sum of the path lengths, and not the maximum path length of $m$ robots.

\subsection{Computational Studies} \label{sec:compstudies}
We implemented the algorithm described in the previous subsection. Our implementation is available online at \url{https://github.com/raaslab/watchman_route} and uses the Noon-Bean implementation from Mathew et al.~\cite{mathew2015multirobot} and VisiLibity library~\cite{visilibity08}. Figure~\ref{fig:noonbean} shows a 2D instance solved using this algorithm. 

The penalties added to the edges can cause their cost to become large enough to run into numerical overflow issues. In our experiments, we encountered instances where the penalty resulted in the edge costs becoming large than what can be represented with the data structure used by concorde. In such a case, we can use the reduction given by Lien and Ma~\cite{Lien1993} which does not require the addition of any penalty. However, their reduction triples the number of nodes in the TSP as compared to the Noon-Bean transformation. This results in larger instances and slower running times. 
Our single robot implementation based on the Lien-Ma transformation is also available online.

An instance with 15 targets and 30 candidate viewpoints for one robot took 41s secs to solve using concorde, where as the same instance took 536s to solve directly using the Integer Linear Programming solver in MATLAB. We use an iterative implementation~\cite{ILPTSP} to find the tour in MATLAB using the ILP function since specifying the full problem directly becomes too large to hold in memory.

\begin{table}
\centering
\begin{tabular}{ |l|c|c|c| } 
 \hline
 Reduction & Solver & Problem Size & Time (secs)\\
 \hline
 Lien-Ma & concorde (DFS) & $|V| = 20$, $|X| = 10$ & 159.97\\
 Lien-Ma & concorde (BFS) & $|V| = 20$, $|X| = 10$ & 225\\
--- & MATLAB ILP & $|V| = 20$, $|X| = 10$ & 40\\
 Noon-Bean & concorde (BFS) & $|V| = 20$, $|X| = 10$ & 2.7\\
 \hline
\end{tabular}
\caption{Time required to solve a representative problem with various implementations.\label{tab:times}}
\end{table}


Table~\ref{tab:times} gives the times required to solve some representative problems using the Lien-Ma and Noon-Bean reductions with concorde and MATLAB. As expected, the Noon-Bean reduction along with the concorde solver is fastest among all options. The Lien-Ma reduction for an instance with 15 targets and 30 candidate viewpoints could not be solved in 12 hours using concorde. An instance with $|V|=63$ and $|X|=15$ could not be solved using MATLAB's ILP function in more than 16 hrs of computation. The same instance took 2411 secs with concorde and Noon-Bean reduction. An instance with $|V|=63$, $|X|=15$, and three robots was solved in 1599 secs with concorde and Noon-Bean.


\begin{figure}[htb]
\centering
\includegraphics[angle=90,width=0.3\textwidth]{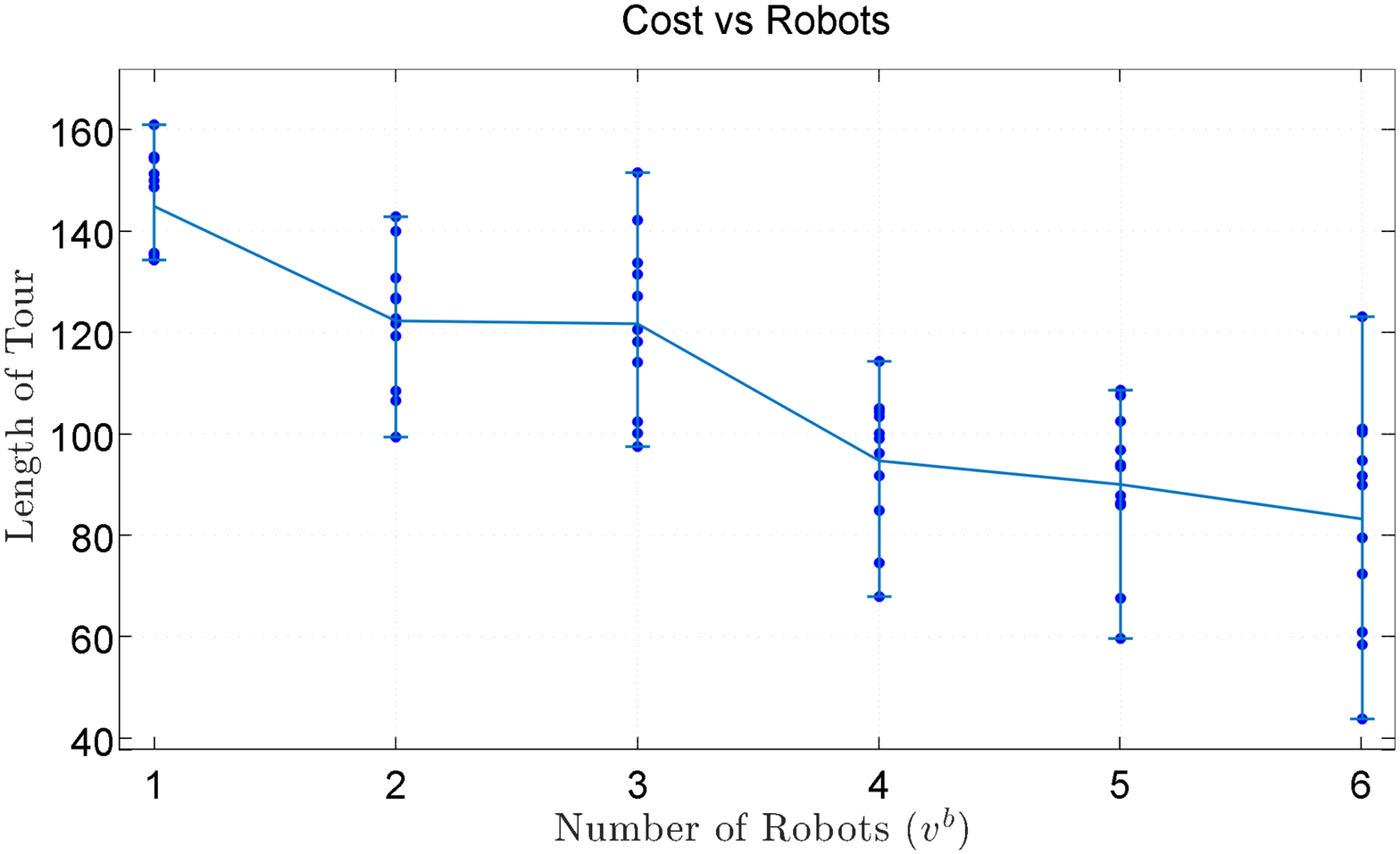}
\includegraphics[angle=90,width=0.3\textwidth]{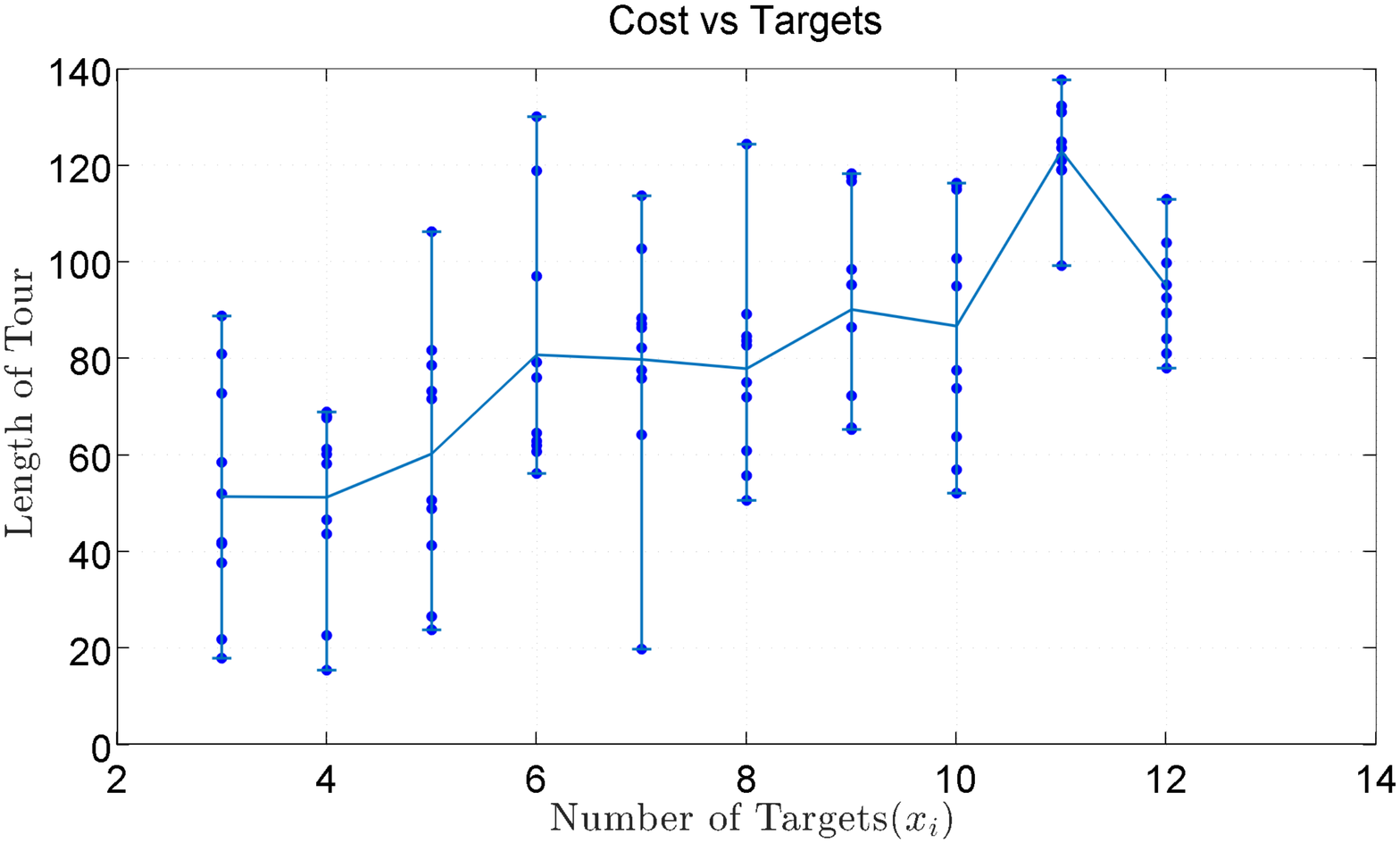}
\includegraphics[angle=90,width=0.3\textwidth]{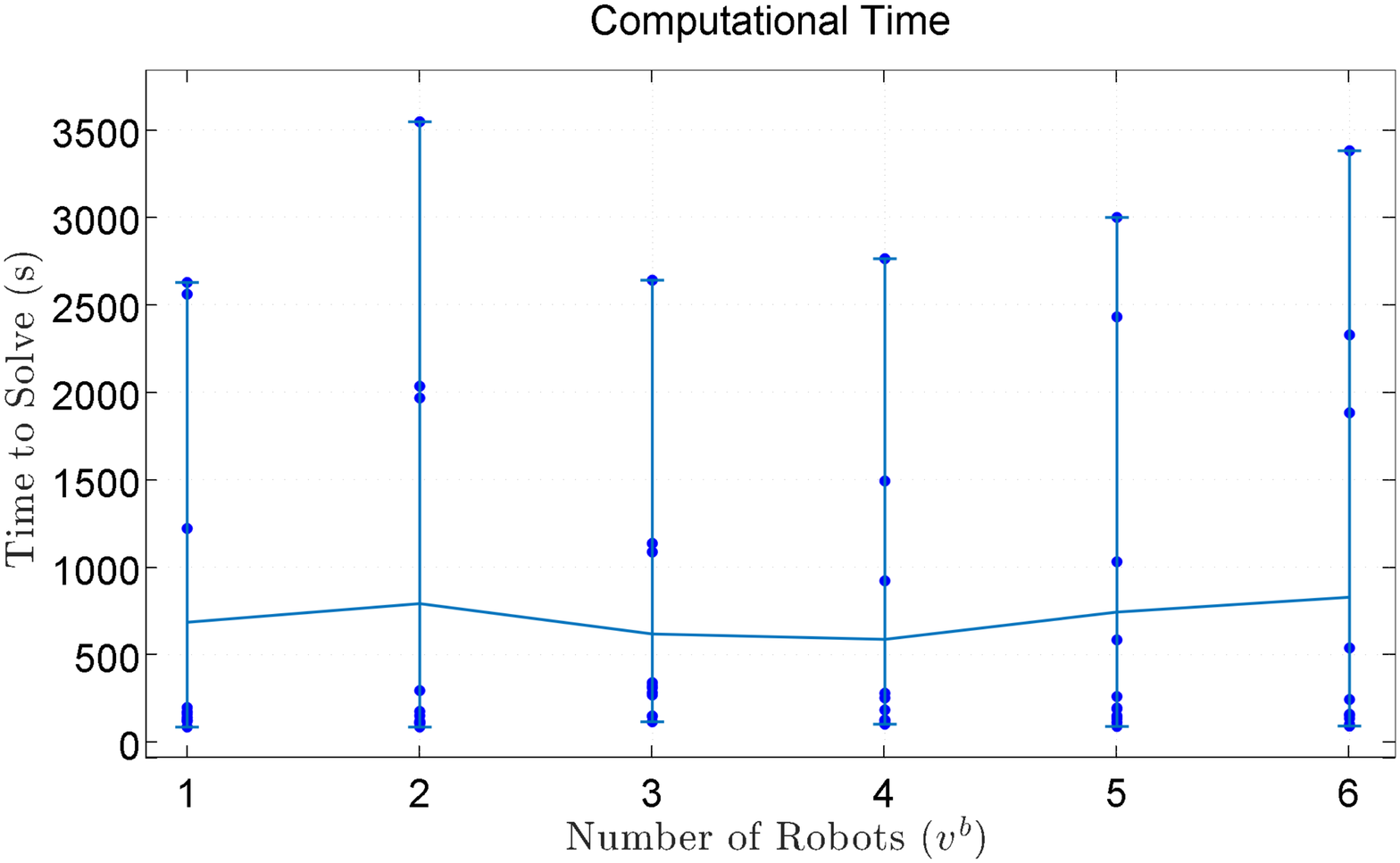}
\includegraphics[angle=90,width=0.3\textwidth]{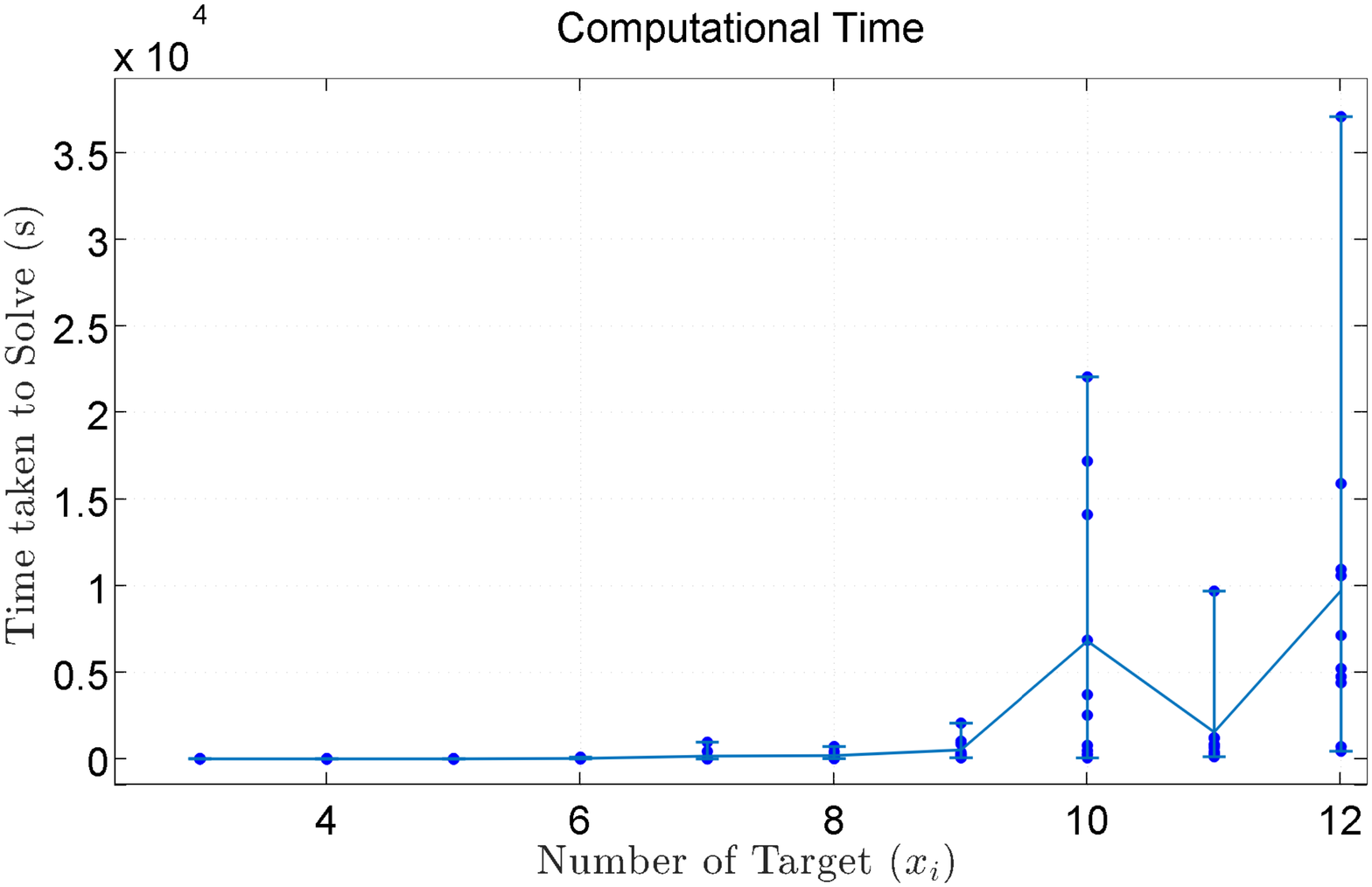}
\caption{(a) Cost of the optimal solution with 10 targets as a function of the number of robots. (b) Cost of the optimal solution with 3 robots as a function of the number of targets. (c) \& (d) Computational time for figures (a) and (b), respectively. The dots show the costs/times of individual random trials and the curve shows the mean of all trials. The robots and target positions were randomly drawn in the environment shown in Figure~\ref{fig:gazebo}.\label{fig:simgtsp}}
\end{figure}

Figure~\ref{fig:simgtsp}(a) shows the effect of varying the number of robots on the optimal cost (sum of path lengths). 10 target locations are randomly generated for each trial in the environment shown in Figure~\ref{fig:gazebo}. Figure~\ref{fig:simgtsp}(b) shows the effect of varying the number of targets. The number of robots were fixed to 3. Figures~\ref{fig:simgtsp}(c)--(d) show the computation time required for finding the solution using the Noon-Bean reduction with concorde (BFS) solver, as a function of the number of robots and the targets. 

The resulting algorithm was also tested in the Gazebo simulation environment (Figure~\ref{fig:gazebo}) using two Pioneer 3DX robots fitted with a limited field-of-view angle camera. The robots emulate an omni-directional camera by rotating in place whenever they reach a new vertex. Table~\ref{tab:gazebo} shows the comparison between the lengths of the tours on the input graph and the actual distance traveled by the robots in the Gazebo simulation environment. The actual distances are shorter since the robot is not restricted to move on the input graph in the polygonal environment.

\begin{table}
\begin{tabular}{|l|c|c|}
\hline
 & Length of Tour Computed & Actual Distance Traveled \\
 \hline
 Robot 1 & 78.06 & 71.74 \\
 Robot 2 & 50.04 & 43.63\\
 \hline
\end{tabular}
\caption{Comparison of the lengths of the tour computed and the actual distance traveled by the robots in the Gazebo simulation environment for Figure~\ref{fig:gazebo}.\label{tab:gazebo}}
\end{table}

\begin{figure}[htb]
\centering{
\label{fig:rviz_fig}\includegraphics[width=0.8\columnwidth]{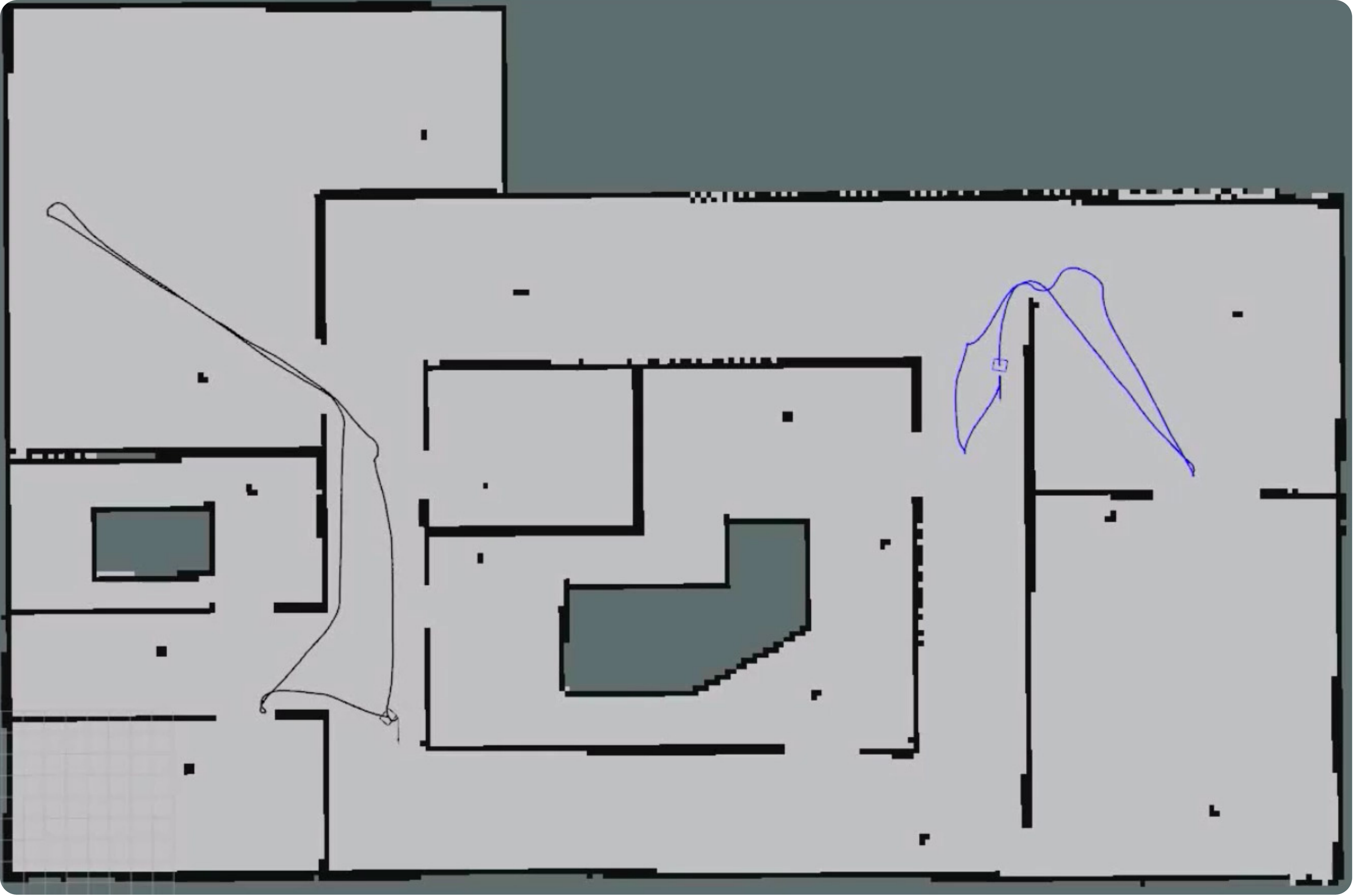}
\caption{Two robot simulation in Gazebo. A video of the simulation is available online: \protect\url{https://youtu.be/lvUyoFZBqxA}. The targets are marked as black dots.\label{fig:gazebo}}
}

\end{figure}


\section{Discussion and Conclusion} \label{sec:conc}
 
Our approach in this paper was to formulate a geometric version of the persistent monitoring problem with a more abstract but generalizable specification. Even with this formulation, the problem turns out to be challenging. 

The analysis presented is based on the key property of
chain-visibility which is satisfied by limited classes of
environments. In particular, for Problem~\ref{prob:street} we require
the environment to be a street polygon. One approach of extending our
algorithm for general environments would be to first decompose the
environment into street polygons, and then apply our algorithm
separately in each component. While algorithms for decomposition into
street polygons are not known, there is an optimal algorithm for
decomposing a polygon without holes into the fewest number of monotone
subpolygons~\cite{keil1985decomposing}. The class of monotone polygons
are included in the class of street polygons, and thus can be used as
valid inputs to our algorithm. Alternatively, human operators may be able to provide this decomposition. However, in both cases, the approximation guarantee will in general not hold. 

%
The main open problem from a theoretical standpoint is whether there
are polygons for which we can compute optimal solutions or
constant-factor approximations for watchmen routes without
requiring a candidate input curve. One possible approach would be to
determine settings in which we can first compute a discrete set of
viewpoints and then find $m$ paths to visit them. There are existing
algorithms for finding $m$ paths of minimum maximum length to visit a
set of points (see e.g.,~\cite{arkin2006approximations}). The key
property would be to show that a tour restricted to the discrete set
of viewpoints thus computed is at most a constant factor away from an
optimal tour. The approximation algorithm presented in
Section~\ref{sec:street} follows this principle. Investigating similar
results for richer environments is part of our ongoing work. 

The algorithm in Section~\ref{sec:general} finds the optimal solution by reducing it to GTSP. However, this is valid when the cost function is the sum of the travel times for all the robots. Extending this approach to account for measurement time, and minimizing the makespan remains an open problem.

\bibliographystyle{IEEEtran}
\bibliography{IEEEabrv,refs}

\appendix
\section*{Proof of Lemma~\ref{lem:dynprogorder}}
\begin{proof}
Consider the case when $v$ is the left endpoint of a path $\Pi_i$ in
an optimal solution. We will prove by contradiction. If $v$ is not
also the right endpoint of some $C_x$, then find the first right
endpoint, say $v'$, of some $C_x$ that is to the right of $v$ along
$C$. All points in $X$ visible from $v$ are also visible from
$v'$. Thus, we can let $v'$ be the new left endpoint of $\Pi_i^*$ to
give a valid solution of lesser length, i.e. lesser cost, which is a
contradiction. The case for the right endpoint is symmetrical.


\end{proof}

\section*{Proof of Lemma~\ref{lem:optsinglepath}}
\begin{proof}
We first verify that all targets in $X'$ will be covered by the
algorithm (and thus the algorithm terminates). Suppose not. Let $x$ be
a target that is not covered. By definition of $X'$, $C_x$ intersects
with $\Pi_i$. Let $x_l$ and $x_r$ be the left and right endpoints of
$C_x$. If $x_l$ is to the left of $i$, then $x$ is visible from $i$
and will be marked covered. If $x_r$ is to the right of $j$, then $x$
is visible from $j$ and will be marked covered. Hence, $x_l$ and $x_r$
lie between $i$ and $j$. 

Consider the closest viewpoint in $V_i$ lying to the left of $x_l$,
say $v$ (we know at least one such viewpoint exists, namely $i$). Let
$v'$ be the first viewpoint in $V$ to the right of $v$ (we know at
least one such viewpoint exists, namely $j$). Now $v'$ cannot be to
the left of $x_l$, else $v$ is not the closest viewpoint to the left
of $x_l$. Similarly, $v'$ cannot be to the right of $x_r$ since $v'$
will not satisfy the condition in Line~\ref{line:q} in
Algorithm~\ref{algo:optsinglepath}. This leaves the case where $v'$ is
between $x_l$ and $x_r$, in which case $x$ is visible from a viewpoint
in $V_i$.

Next, we verify that the optimal set of viewpoints and cost is
correctly computed. The length of $\Pi_i$ is fixed since $i$ and $j$
are given as input. Let $X''$ be the subset of $X'$ such that any
$x\in X''$ is not visible from either $i$ or $j$. It remains to show
that $|V_i\setminus\{i,j\}|$ is the least number of measurements
required to cover $X''$. Denote the viewpoints in
$V_i\setminus\{i,j\}$ by $v_1,\ldots,v_n$. Along with $i$ and $j$,
this defines $n+1$ partitions: $[v_0:=i,v_1],
[v_1,v_2],\ldots,[v_n,v_{n+1}:=j]$. 

For contradiction, suppose there is a $V'=\{v_i'\}$ with $n-1$
viewpoints that cover $X''$. Then, there must exist at least one
$[v_i,v_{i+1}]$ partition that does not contain any $v_j'$. From
Line~\ref{line:q}, this implies that there is some point $x$ whose
interval lies completely between two consecutive viewpoints in
$V'$. Thus, $V'$ does not cover all elements in $X$ which is a
contradiction.

\end{proof}

\end{document}